\documentclass{article}

     \usepackage[preprint,nonatbib]{neurips_2019}

\usepackage[utf8]{inputenc} 
\usepackage[T1]{fontenc}    
\usepackage{hyperref} 
\usepackage{url}            
\usepackage{booktabs}       
\usepackage{amsfonts}       
\usepackage{nicefrac}       
\usepackage{microtype}      
\usepackage{amssymb,amsthm,amsfonts,amscd}
\usepackage{graphicx}
\usepackage{amsmath}
\usepackage{mathtools}
\usepackage{xcolor}
\usepackage{natbib}
\usepackage{booktabs}
\usepackage{array}
\usepackage{wrapfig}
\usepackage{multirow}
\usepackage{tabu}
\usepackage{makecell}

\usepackage{thmtools}
\usepackage{thm-restate}
\usepackage{cleveref}
\usepackage{algorithm}
\usepackage{algpseudocode}
\usepackage{titletoc}
\usepackage{etoc}
\usepackage{authblk}
\newtheorem{thm}{Theorem}
\newtheorem{prop}[thm]{Proposition}
\newtheorem{lma}[thm]{Lemma}
\newtheorem{cor}[thm]{Corollary}

\newcolumntype{C}[1]{>{\centering\arraybackslash}m{#1}}
\newcommand{\del}[1]{ \Delta^{\rm EP}_{#1} }
\newcommand{\nab}[1]{ \nabla^{\rm BPTT}_{#1} }
\newcommand{\norm}[1]{\left\lVert #1\right \rVert}  
\DeclareMathOperator*{\argmax}{arg\,max}

\DeclarePairedDelimiter{\ceil}{\lceil}{\rceil}

\title{Updates of Equilibrium Prop Match Gradients of Backprop Through Time in an RNN with Static Input}

\author{
  {\bfseries \small Maxence Ernoult$^{1,2}$, Julie Grollier$^2$, Damien Querlioz$^1$, Yoshua Bengio$^{3,4}$, Benjamin Scellier$^3$}\\
  
  {\small $^1$Centre de Nanosciences et de Nanotechnologies, Université Paris Sud, Université Paris-Saclay\\
  $^2$Unité Mixte de Physique, CNRS, Thales, Université Paris-Sud, Université Paris-Saclay\\
  $^3$Mila, Université de Montréal\\
  $^4$Canadian Institute for Advanced Research}
}

\hypersetup{draft}
\begin{document}

\maketitle

\begin{abstract}
Equilibrium Propagation (EP) is a  biologically inspired learning algorithm for convergent recurrent neural networks, 
i.e. RNNs that are fed by a static input $x$ and settle to a steady state. 
Training convergent RNNs consists in adjusting the weights until the steady state of output neurons coincides with a target $y$. 
Convergent RNNs can also be trained with the more conventional Backpropagation Through Time (BPTT) algorithm.
In its original formulation EP was described in the case of real-time neuronal dynamics, which is computationally costly.
In this work, we introduce a discrete-time version of EP
with simplified equations and with reduced simulation time,
bringing EP closer to practical machine learning tasks.  
We first prove theoretically, as well as numerically that the neural and weight updates of EP,
computed by \emph{forward-time} dynamics,
are step-by-step equal to the ones obtained by BPTT,
with gradients computed \emph{backward in time}. 
The equality is strict when
the transition function of the dynamics
derives from a primitive function and the steady state is maintained long enough.
We then show for more standard discrete-time neural network dynamics
that the same property is approximately respected and we subsequently demonstrate training with EP with equivalent performance to BPTT. 
In particular, we define the first convolutional architecture trained with EP achieving $\sim 1\%$ test error on MNIST, which is the lowest
error reported with EP.
These results can guide the development of deep neural networks trained with EP.
\end{abstract}

\section{Introduction}


The remarkable development of deep learning over the past years \citep{lecun2015deep} has been fostered by the use of backpropagation \citep{rumelhart1985learning} which stands as the most powerful algorithm to train neural networks.
In spite of its success, the backpropagation algorithm is not biologically plausible \citep{crick-nature1989},
and its implementation on GPUs is energy-consuming \citep{editorial_big_2018}.
Hybrid hardware-software experiments have recently demonstrated how physics and dynamics can be leveraged to achieve learning with energy efficiency \citep{romera2018vowel, ambrogio2018equivalent}.
Hence the motivation to invent novel learning algorithms
where both inference and learning could fully be achieved out of core physics.

Many biologically inspired learning algorithms have been proposed as alternatives to backpropagation to train neural networks.
Contrastive Hebbian learning (CHL) has been successfully used to train recurrent neural networks (RNNs) with static input that converge to a steady state (or `equilibrium'), such as Boltzmann machines \citep{ackley1985learning} and real-time Hopfield networks \citep{movellan1991contrastive}.
CHL proceeds in two phases, each phase converging to a steady state, where the learning rule accommodates the difference between the two equilibria.
Equilibrium Propagation (EP) \citep{Scellier+Bengio-frontiers2017} also belongs to the family of CHL algorithms to train RNNs with static input.
In the second phase of EP, the prediction error is encoded as an elastic force nudging the system towards a second equilibrium closer to the target.
Interestingly, EP also shares similar features with the backpropagation algorithm, and more specifically recurrent backpropagation (RBP) \citep{Almeida87,Pineda87}. 
It was proved in \citet{scellier2019equivalence} that neural computation in the second phase of EP is equivalent to gradient computation in RBP. 

Originally, EP was introduced in the context of real-time leaky integrate neuronal dynamics \citep{Scellier+Bengio-frontiers2017,scellier2019equivalence} whose computation involves long simulation times, hence limiting EP training experiments to small neural networks.
In this paper, we propose a discrete-time formulation of EP. 
This formulation allows demonstrating an equivalence between EP and BPTT in specific conditions, simplifies equations and speeds up training, 
and extends EP to standard neural networks including convolutional ones.
Specifically, the contributions of the present work are the following:

\begin{itemize}
    \item We introduce a discrete-time formulation of EP of which the original real-time formulation can be seen as a particular case (Section \ref{sec:ep-algo}).
    \item We show a step-by-step equality between the updates of EP and the gradients of BPTT
    when the dynamics converges to a steady state
    and the transition function of the RNN derives from a primitive function (Theorem~\ref{thm:main}, Figure~\ref{fig:theorem}).
    We say that such an RNN has the property of `gradient-descending updates' (or GDU property).
    \item We numerically demonstrate the GDU property on a small network, on fully connected layered and convolutional architectures. We show that the GDU property continues to hold approximately for more standard -- \emph{prototypical} -- neural networks even if these networks do not exactly meet the requirements of Theorem~\ref{thm:main}.
    \item We validate our approach with training experiments on different network architectures using discrete-time EP, achieving similar performance than BPTT. We show that the number of iterations in the two phases of discrete-time EP can be reduced by a factor three to five compared to the original real-time EP, without loss of accuracy. This allows us training the first convolutional architecture with EP, reaching $\sim 1\%$ test error on MNIST, which is the lowest test error reported with EP. 
\end{itemize}

\begin{figure}[ht!]
\begin{center}
   \fbox{\includegraphics[width=0.9\textwidth]{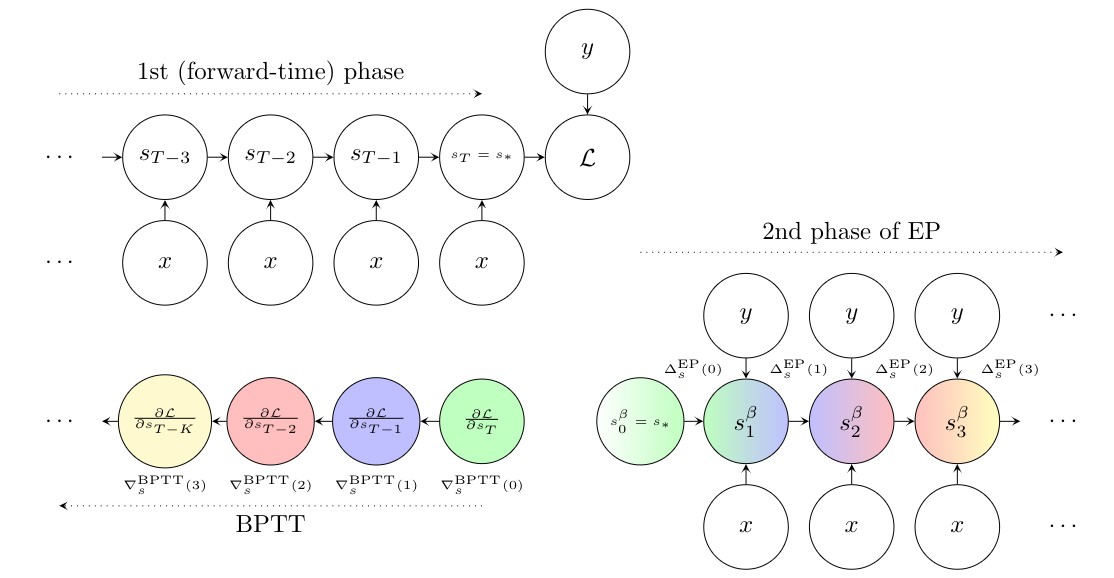}}
\end{center}
  \caption{Illustration of the property of Gradient-Descending Updates (GDU property). \textbf{Top left.} Forward-time pass (or `first phase') of an RNN with static input $x$ and target $y$. The final state $s_T$ is the steady state $s_*$. \textbf{Bottom left.} Backprop through time (BPTT). \textbf{Bottom right.} Second phase of equilibrium prop (EP). The starting state in the second phase is the final state of the first phase, i.e. the steady state $s_*$. \textbf{GDU Property (Theorem \ref{thm:main}).} Step by step correspondence between the neural updates $\Delta_s^{\rm EP}(t)$ in the second phase of EP and the gradients $\nabla_s^{\rm BPTT}(t)$ of BPTT. Corresponding computations in EP and BPTT at timestep $t=0$ (resp. $t=1,2,3$) are colored in green (resp. blue, red, yellow). Forward-time computation in EP corresponds to backward-time computation in BPTT.}
  \label{fig:theorem}
\end{figure}

\section{Background}

This section introduces the notations and basic concepts used throughout the paper.

\subsection{Convergent RNNs With Static Input}

We consider the supervised setting where we want to predict a target $y$ given an
input $x$.
The model is a dynamical system - such as a recurrent neural network (RNN) - parametrized by $\theta$ and evolving according to the dynamics:
\begin{equation}
    s_{t+1} = F \left( x,s_t,\theta \right).
    \label{eq:1st-phase}
\end{equation}
We call $F$ the \textit{transition function}.
The input of the RNN at each timestep is static, equal to $x$.
Assuming convergence of the dynamics before time step $T$, we have $s_T = s_*$ where $s_*$ is such that
\begin{equation}
    s_* = F \left( x,s_*,\theta \right).
\end{equation}
We call $s_*$ the \textit{steady state} (or fixed point, or equilibrium state) of the dynamical system.
The number of timesteps $T$ is a hyperparameter chosen large enough to ensure $s_T = s_*$.
The goal of learning is to optimize the parameter $\theta$
to minimize the loss:
\begin{equation}
	{\mathcal L^*} = \ell \left( s_*, y \right),
\end{equation}
where the scalar function $\ell$ is called  \textit{cost function}.
Several algorithms have been proposed to optimize the loss ${\mathcal L^*}$, including Recurrent Backpropagation (RBP) \citep{Almeida87,Pineda87} and Equilibrium Propagation (EP) \citep{Scellier+Bengio-frontiers2017}. 
Here, we present Backpropagation Through Time (BPTT) and Equilibrium Propagation (EP) and some of the inner mechanisms of these two algorithms, so as to enunciate the main theoretical result of this paper (Theorem~\ref{thm:main}).

\subsection{Backpropagation Through Time (BPTT)}
\label{sec:background-bptt}

With frameworks implementing automatic differentiation, optimization by gradient descent using Backpropagation Through Time (BPTT) has become the standard method to train RNNs.
In particular BPTT can be used for a convergent RNN such as the one that we study here.
To this end, we consider the loss after $T$ iterations (i.e. the cost of the final state $s_T$), denoted ${\mathcal L} = \ell \left( s_T, y \right)$, and we substitute ${\mathcal L}$ as a proxy
\footnote{The difference between the loss ${\mathcal L}$ and the loss ${\mathcal L}^*$ is explained in Appendix \ref{sec:loss-L-L*}.}
for the loss at the steady state ${\mathcal{L}^*}$.
The gradients of ${\mathcal L}$ can be computed with BPTT.

In order to state our Theorem \ref{thm:main} (Section \ref{sec:theorem}), we recall some of the inner working mechanisms of BPTT.
Eq.~(\ref{eq:1st-phase}) 
can be rewritten
in the form $s_{t+1} = F \left( x,s_t,\theta_{t+1} = \theta \right)$, where $\theta_t$ denotes the parameter of the model at time step $t$, the value $\theta$ being shared across all time steps.
This way of rewriting Eq.~(\ref{eq:1st-phase}) enables us to define the partial derivative $\frac{\partial \mathcal{L}}{\partial \theta_t}$ as the sensitivity of the loss $\mathcal{L}$ with respect to $\theta_t$ when $\theta_1, \ldots \theta_{t-1}, \theta_{t+1}, \ldots \theta_T$ remain fixed (set to the value $\theta$). With these notations, the gradient $\frac{\partial \mathcal{L}}{\partial \theta}$ reads as the sum:
\begin{equation}
\label{eq:total-gradient}
\frac{\partial {\mathcal L}}{\partial \theta} = \frac{\partial {\mathcal L}}{\partial \theta_1} + \frac{\partial {\mathcal L}}{\partial \theta_2} + \cdots + \frac{\partial {\mathcal L}}{\partial \theta_T}.
\end{equation}

BPTT computes the `full' gradient $\frac{\partial {\mathcal L}}{\partial \theta}$ by computing the partial derivatives $\frac{\partial {\mathcal L}}{\partial s_t}$ and $\frac{\partial {\mathcal L}}{\partial \theta_t}$ iteratively and efficiently, backward in time, using the chain rule of differentiation.
Subsequently, we denote the gradients that BPTT computes:
\begin{align}
\forall t \in [0, T - 1]:
\left\{
\begin{array}{ll}
   \displaystyle \nab{s}(t) = \frac{\partial {\mathcal L}}{\partial s_{T-t}}\\
 \displaystyle \nab{\theta}(t)  = \frac{\partial {\mathcal L}}{\partial \theta_{T-t}},
\end{array}
\right.
\end{align}
so that
\begin{equation}
\frac{\partial \mathcal{L}}{\partial \theta} = \sum_{t=0}^{T-1} \nabla_{\theta}^{\rm BPTT}(t).
\label{bptt-grad}
\end{equation}

More details about BPTT are provided in Appendix~\ref{sec:bptt}.

\section{Equilibrium Propagation (EP) - Discrete Time Formulation}
\label{sec:ep}

\subsection{Algorithm}
\label{sec:ep-algo}

In its original formulation, Equilibrium Propagation (EP) was introduced in the case of real-time dynamics \citep{Scellier+Bengio-frontiers2017,scellier2019equivalence}.
The first theoretical contribution of this paper is to adapt the theory of EP to discrete-time dynamics.\footnote{We explain in Appendix~\ref{sec:primitive-energy} the relationship between the discrete-time setting (resp. the primitive function $\Phi$) of this paper and the real-time setting (resp. the energy function $E$) of \citet{Scellier+Bengio-frontiers2017,scellier2019equivalence}.}
EP is an alternative algorithm to compute the gradient of ${\mathcal L}^*$ in the particular case where the transition function $F$ derives from a scalar function $\Phi$, i.e. with $F$ of the form $F(x,s,\theta) = \frac{\partial \Phi}{\partial s}(x,s,\theta)$.
In this setting, the dynamics of Eq.~(\ref{eq:1st-phase}) rewrites:
\begin{equation}
\forall t \in [0, T-1], \qquad s_{t+1} = \frac{\partial \Phi}{\partial s}(x, s_t, \theta).
\end{equation}
This constitutes the first phase of EP.
At the end of the first phase,
we have reached steady state, i.e. $s_T = s_*$.
In the second phase of EP, starting from the steady state $s_*$, an extra term $\beta \; \frac{\partial \ell}{\partial s}$ (where $\beta$ is a positive scaling factor) is introduced in the dynamics of the neurons and acts as an external force nudging the system dynamics towards decreasing the cost function $\ell$. Denoting $s_0^\beta, s_1^\beta, s_2^\beta, \ldots$ the sequence of states in the second phase (which depends on the value of $\beta$), the dynamics is defined as
\begin{equation}
	s_0^\beta = s_* \qquad \text{and} \qquad \forall t \geq 0, \quad s_{t+1}^\beta  = \frac{\partial \Phi}{\partial s} \left( x,s_t^\beta,\theta \right) - \beta \; \frac{\partial \ell}{\partial s} \left( s_t^\beta,y \right).
\label{eq:2nd-phase}	
\end{equation}
The network eventually settles to a new steady state $s_*^\beta$.
It was shown in \citet{Scellier+Bengio-frontiers2017} that the gradient of the loss ${\mathcal L}^*$ can be computed based on the two steady states $s_*$ and $s_*^\beta$. 
More specifically,
in the limit $\beta \to 0$,
\begin{equation}
 \frac{1}{\beta} \left( \frac{\partial \Phi}{\partial \theta} \left( x, s_*^\beta,\theta \right) -  \frac{\partial \Phi}{\partial \theta} \left( x,s_*,\theta \right) \right) \to -\frac{\partial {\mathcal L}^*}{\partial \theta}.
\label{eq:ep-grad}
\end{equation}

In fact, we can prove a stronger result.
For fixed $\beta > 0$ we define the neural and weight updates
\begin{align}
\forall t \geq 0:
\left\{
\begin{array}{ll}
    \displaystyle \del{s}(\beta,t) &= \frac{1}{\beta} \left( s_{t+1}^\beta-s_t^\beta \right), \\
   \displaystyle \del{\theta}(\beta,t) &= \frac{1}{\beta} \left( \frac{\partial \Phi}{\partial \theta} \left( x,s_{t+1}^\beta,\theta \right) -  \frac{\partial \Phi}{\partial \theta} \left( x,s_t^\beta,\theta \right) \right),
\end{array}
\right.
\label{eq:delta-ep}
\end{align}
and note that Eq.~(\ref{eq:ep-grad}) rewrites as the following telescoping sum:
\begin{equation}
\sum_{t =0}^\infty \del{\theta}(\beta,t) \to - \frac{\partial \mathcal{L}^*}{\partial \theta} \qquad \text{as} \qquad \beta \to 0.
\label{eq:ep-grad2}
\end{equation}
We can now state our main theoretical result
(Theorem \ref{thm:main} below, proved in Appendix \ref{sec:bptt-ep}).

\subsection{Forward-Time Dynamics of EP Compute Backward-Time Gradients of BPTT}
\label{sec:theorem}

BPTT and EP compute the gradient of the loss in very different ways: while the former algorithm iteratively adds up gradients going backward in time, as in Eq.~(\ref{bptt-grad}), the latter algorithm adds up weight updates going forward in time, as in Eq.~(\ref{eq:ep-grad2}).
In fact, under a condition stated below, the sums are equal term by term: there is a step-by-step correspondence between the two algorithms. 

\begin{restatable}[Gradient-Descending Updates, GDU]{thm}{theor}
\label{thm:main}
In the setting with a transition function of the form $F \left( x,s,\theta \right) = \frac{\partial \Phi}{\partial s} \left( x,s,\theta \right)$,
let $s_0, s_1, \ldots, s_T$ be the convergent sequence of states and denote $s_* = s_T$ the steady state.
If we further assume that there exists some step $K$ where $0 < K \leq T$ such that $s_* = s_T = s_{T-1} = \ldots s_{T-K}$, then, in the limit $\beta \to 0$, the first $K$ updates in the second phase of EP are equal to the negatives of the first $K$ gradients of BPTT, i.e.
\begin{equation}
    \forall t=0,1,\ldots,K, \qquad \del{s}(\beta,t) \to - \nab{s}(t) \qquad \text{and} \qquad \del{\theta}(\beta,t) \to - \nab{\theta}(t).
    \label{eq:thm-main}
\end{equation}
\end{restatable}


\section{Experiments}

This section uses Theorem~\ref{thm:main} as a tool to design neural networks that are trainable with EP: if a model satisfies the GDU property of Eq.~\ref{eq:thm-main}, then we expect EP to perform as well as BPTT on this model.
After introducing our protocol (Section~\ref{sec:gradient-matching}), we define the \textit{energy-based setting} and \textit{prototypical setting} where the conditions of Theorem~\ref{thm:main} are exactly and approximately met respectively (Section~\ref{sec:eb-prototypical-settings}). We show the GDU property on a toy model (Fig.~\ref{fig:toy-model}) and on fully connected layered architectures in the two settings (Section \ref{sec:depth}). We define a convolutional architecture in the prototypical setting (Section \ref{subsec-conv}) which also satisfies the GDU property. Finally, we validate our approach by training these models with EP and BPTT (Table \ref{table-results}).

\begin{figure}[ht!]
\begin{center}
   \includegraphics[scale=0.42]{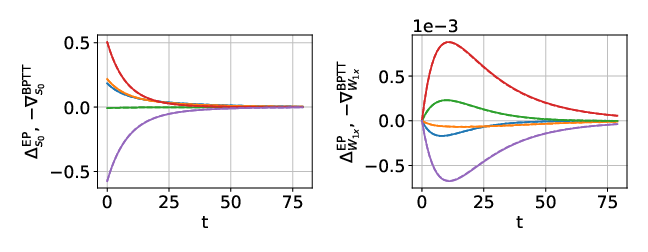}
      \hfill 
   \includegraphics[scale=0.17]{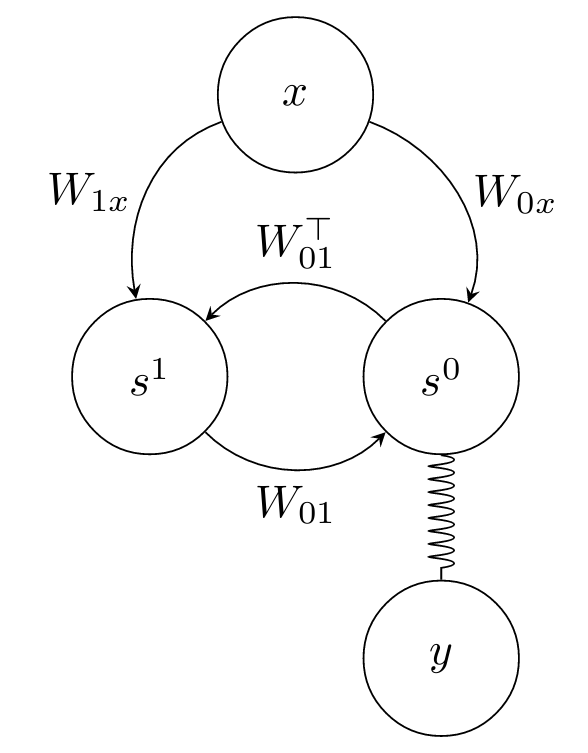}
\end{center}
\caption{Demonstrating the property of gradient-descending updates in the energy-based setting on a toy model with dummy data $x$ and a target $y$ elastically nudging the output neurons $s^0$ (right).
Dashed and solid lines represent $\Delta^{\rm EP}$ and $-\nabla^{\rm BPTT}$ processes respectively and perfectly coincide for 5 randomly selected neurons (left) and synapses (middle).
Each randomly selected neuron or synapse corresponds to one color. Details can be found in Appendix~\ref{exp:toymodel}.
}
\label{fig:toy-model}
\end{figure}

\subsection{Demonstrating the Property of Gradient-Descending Updates (GDU Property)}
\label{sec:gradient-matching}

\paragraph{Property of Gradient-Descending Updates.} We say that a convergent RNN model fed with a fixed input has the \textit{GDU property} if during the second phase, the updates it computes by EP ($\del{}$) on the one hand and the gradients it computes by BPTT ($-\nab{}$) on the other hand are `equal' -  or `approximately equal', as measured per the RMSE (Relative Mean Squared Error) metric.

\paragraph{Relative Mean Squared Error (RMSE).} In order to quantitatively measure how well the GDU property is satisfied, we introduce a metric which we call Relative Mean Squared Error (RMSE) such that RMSE($\Delta^{\rm EP}$, -$\nabla^{\rm BPTT}$) measures the distance between $\Delta^{\rm EP}$ and $-\nabla^{\rm BPTT}$ processes, averaged over time, over neurons or synapses (layer-wise) and over a mini-batch of samples - see Appendix~\ref{exp:rmse} for the details.

\paragraph{Protocol.} In order to measure numerically if a given model satisfies the GDU property, we proceed as follows. Considering an input $x$ and associated target $y$, we perform the first phase for T steps. Then we perform on the one hand BPTT for $K$ steps (to compute the gradients $\nabla^{\rm BPTT}$), on the other hand EP for $K$ steps (to compute the neural updates $\Delta^{\rm EP}$) and compare the gradients and neural updates provided by the two algorithms, either qualitatively by looking at the plots of the curves (as in Figs.~\ref{fig:toy-model} and \ref{th-conv}), or quantitatively by computing their RMSE (as in Fig.~\ref{RMSE}).

\subsection{Energy-Based Setting and Prototypical Setting}
\label{sec:eb-prototypical-settings}

\paragraph{Energy-based setting.} The system is defined in terms of a primitive function of the form:
\begin{equation}
\Phi_\epsilon(s,W) = (1-\epsilon)\frac{1}{2}\|s\|^2 + \epsilon \; \sigma(s)^\top \cdot W \cdot \sigma(s),
\label{eq:energy-based}
\end{equation}
where $\epsilon$ is a discretization parameter, $\sigma$ is an activation function and $W$ is a symmetric weight matrix.
In this setting, we consider $\del{}(\beta \epsilon, t)$ instead of $\del{}(\beta, t)$ and write $\del{}(t)$ for simplicity, so that:
\begin{equation}
    \del{s}(t) = \frac{s_{t+1}^{\beta \epsilon}-s_t^{\beta \epsilon}}{\beta\epsilon}, \quad
    \del{W}(t) = \frac{1}{\beta} \left( \sigma \left( s_{t+1}^{\beta \epsilon} \right)^\top \cdot \sigma \left( s_{t+1}^{\beta \epsilon} \right) - \sigma \left( s_t^{\beta \epsilon} \right)^\top \cdot \sigma \left( s_t^{\beta \epsilon} \right) \right).
   \label{delta-fc-canon}
\end{equation}
With $\Phi_\epsilon$ as a primitive function and with the hyperparameter $\beta$ rescaled by a factor $\epsilon$, we recover the discretized version of the real-time setting of \citet{Scellier+Bengio-frontiers2017}, i.e. the Euler scheme of $\frac{ds}{dt}=-\frac{\partial E}{\partial s} - \beta \frac{\partial \ell}{\partial s}$ with $E = \frac{1}{2}\|s\|^2 - \sigma(s)^\top \cdot W \cdot \sigma(s)$ -- see Appendix \ref{sec:primitive-energy}.
Fig.~\ref{fig:toy-model} qualitatively demonstrates Theorem~\ref{thm:main} in this setting on a toy model.

\paragraph{Prototypical Setting.} In this case, the dynamics of the system does not derive from a primitive function $\Phi$. Instead, the dynamics is directly defined as: 
\begin{equation}
s_{t+1} = \sigma \left( W \cdot s_t \right).
\label{def-prototypical}
\end{equation}
Again, $W$ is assumed to be a symmetric matrix.
The dynamics of Eq.~(\ref{def-prototypical}) is a standard and simple neural network dynamics.
Although the model is not defined in terms of a primitive function, note that $s_{t+1} \approx \frac{\partial \Phi}{\partial s} \left( s_t,W \right)$ with $\Phi(s,W) = \frac{1}{2} s^\top \cdot W \cdot s$ if we ignore the activation function $\sigma$.
Following Eq.~(\ref{eq:delta-ep}), we define:
\begin{equation}
    \del{s}(t) = \frac{1}{\beta} \left( s_{t+1}^\beta-s_t^\beta \right), \qquad
    \del{W}(t) = \frac{1}{\beta} \left( s_{t+1}^{\beta^\top} \cdot s_{t+1}^\beta - s_t^{\beta^\top} \cdot s_t^\beta \right).
   \label{delta-fc}
\end{equation}

\subsection{Effect of Depth and Approximation}
\label{sec:depth}

We consider a fully connected layered architecture where layers $s^n$ are labelled in a backward fashion: $s^0$ denotes the output layer, $s^1$ the last hidden layer, and so forth. Two consecutive layers are reciprocally connected with tied weights with the convention that $W_{n,n+1}$ connects $s^{n+1}$ to $s^{n}$.
We study this architecture in the energy-based and prototypical setting as described per Equations (\ref{eq:energy-based}) and (\ref{def-prototypical}) respectively - see details in Appendix \ref{exp:def-fc-eb} and \ref{exp:def-fc-p}.
We study the GDU property layer-wise, e.g.
RMSE($\Delta_{s^{n}}^{\rm EP}$, -$\nabla_{s^n}^{\rm BPTT}$) measures the distance between the $\Delta_{s^{n}}^{\rm EP}$ and $-\nabla_{s^n}^{\rm BPTT}$ processes, averaged over all elements of layer $s^{n}$. 

\begin{figure}[ht!]
\begin{center}
   \includegraphics[scale=0.4]{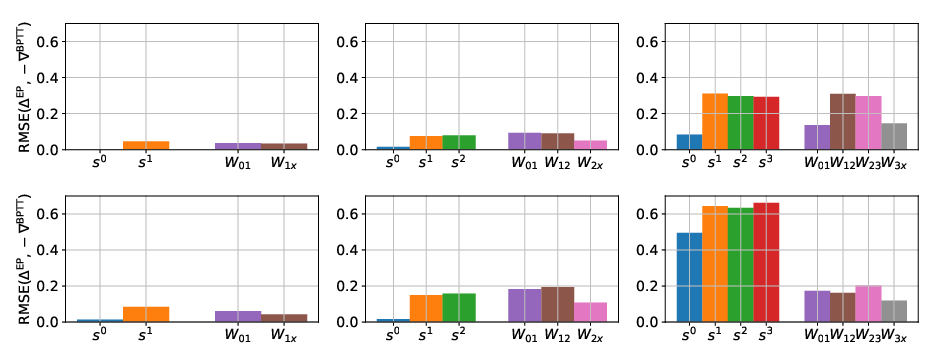}
\end{center}
  \caption{RMSE analysis in the energy-based (top) and prototypical (bottom) setting. For one given architecture, each bar is labelled by a layer or synapses connecting two layers, e.g. the orange bar above $s^1$ represents ${\rm RMSE}(\Delta_{s^1}^{\rm EP}, -\nabla_{s^{1}}^{\rm BPTT})$. For each architecture, the recurrent hyperparameters $T$, $K$ and $\beta$ have been tuned to make the $\Delta^{\rm EP}$ and $-\nabla^{\rm BPTT}$ processes match best.}
\label{RMSE}
\end{figure}

We display in Fig.~\ref{RMSE} the RMSE, layer-wise for one, two and three hidden layered architecture (from left to right), in the energy-based (upper panels) and prototypical (lower panels) settings, so that each architecture in a given setting is displayed in one panel - see Appendix \ref{exp:def-fc-eb} and \ref{exp:def-fc-p} for a detailed description of the hyperparameters and curve samples. In terms of RMSE, we can see that the GDU property is best satisfied in the energy-based setting with one hidden layer where RMSE is around $\sim 10^{-2}$ (top left). When adding more hidden layers in the energy-based setting (top middle and top right), the RMSE increases to $\sim 10^{-1}$, with a greater RMSE when going away from the output layer.
The same is observed in the prototypical setting when we add more hidden layers (lower panels).
Compared to the energy-based setting, although the RMSEs associated with neurons are significantly higher in the prototypical setting, the RMSEs associated with synapses are similar or lower.
On average, the weight updates provided by EP match well the gradients of BPTT, in the energy-based setting as well as in the prototypical setting.

\subsection{Convolutional Architecture}
\label{subsec-conv}

\begin{figure}[ht!]
\begin{center}
\includegraphics[scale=0.46]{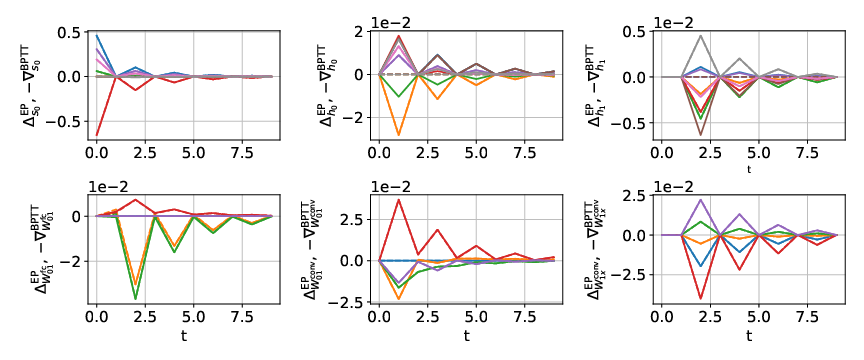}
\end{center}
\caption{Demonstrating the GDU property with the convolutional architecture on MNIST. Dashed and continuous lines represent $\Delta^{\rm EP}$ and $-\nabla^{\rm BPTT}$ processes respectively, for 5 randomly selected neurons (top) and synapses (bottom) in each layer.
Each randomly selected neuron or synapse corresponds to one color.
Dashed and continuous lines mostly coincide.
Some $\Delta^{\rm EP}$ processes collapse to zero as an effect of the non-linearity, see Appendix \ref{exp:conv} for details. Interestingly, the $\del{s}$ and $-\nab{s}$ processes are saw-teeth-shaped ; Appendix~\ref{subsec-sawteeth} accounts for this phenomenon.}
\label{th-conv}
\end{figure}

In our convolutional architecture, $h^n$ and $s^n$ denote convolutional and fully connected layers respectively. $W^{\rm fc}_{n, n+1}$ and $W^{\rm conv}_{n, n+1}$ denote the fully connected weights connecting $s^{n+1}$ to $s^n$ and the filters connecting $h^{n+1}$ to $h^n$, respectively.
We define the dynamics as:
\begin{align}
\left\{
\begin{array}{ll}
s^n_{t + 1} &= \sigma \left(W^{\rm fc}_{n n+1}\cdot s^{n+1}_{t} + W^{{\rm fc}^\top}_{n -1 n}\cdot s^{n-1}_{t} \right)\\
h^{n}_{t+1} &= \sigma \left( \mathcal{P}\left(W^{\rm conv}_{n, n+1}*h^{n+1}_t\right) + \tilde{W}_{n-1, n}^{\rm conv}*\mathcal{P}^{-1}\left(h^{n-1}_{t}\right)\right),
\end{array} 
\right.
\end{align}
where $*$ and $\mathcal{P}$ denote convolution and pooling, respectively. Transpose convolution is defined through the convolution by the flipped kernel $\tilde{W}^{\rm conv}$ and $\mathcal{P}^{-1}$ denotes inverse pooling - see Appendix~\ref{exp:conv} for a precise definition of these operations and their inverse.
Noting $N_{\rm fc}$ and $N_{\rm conv}$ the number of fully connected and convolutional layers respectively, we can define the function:
\begin{equation}
    \Phi(x, \{s^{n}\}, \{h^{n}\}) = \sum_{n = 0}^{N_{\rm conv} - 1} h^{n}\bullet\mathcal{P}\left(W_{n, n+1}^{\rm conv}*h^{n+1}\right) + \sum_{n=0}^{N_{\rm fc - 1}} s^{n \top}\cdot W_{n, n+1}^{\rm fc}\cdot s^{n+1}, 
\end{equation}
with $\bullet$ denoting generalized scalar product. We note that $s^n_{t+1} \approx \frac{\partial \Phi}{\partial s}(t)$ and $h^n_{t+1} \approx \frac{\partial \Phi}{\partial h}(t)$ if we ignore the activation function $\sigma$.
We define $\Delta_{s^n}^{\rm EP}$, $\Delta_{h^n}^{\rm EP}$ and $\Delta_{W^{\rm fc}}^{\rm EP}$ as in Eq.~\ref{delta-fc}.
As for $\Delta_{W^{\rm conv}}^{\rm EP}$, we follow the definition of Eq.~(\ref{eq:delta-ep}):
\begin{equation}
\Delta_{W_{n n+1}^{\rm conv}}^{\rm EP}(t)  =  \frac{1}{\beta} \left(\mathcal{P}^{-1}(h^{n, \beta}_{t+1})*h^{n+1, \beta}_{t+1} - \mathcal{P}^{-1}(h^{n, \beta}_{t})*h^{n+1, \beta}_{t} \right)  
\label{deltaconv-main}
\end{equation}
As can be seen in Fig.~\ref{th-conv}, the GDU property is qualitatively very well satisfied: Eq.~(\ref{deltaconv-main}) can thus be safely used as a learning rule. More precisely however, some $\Delta^{\rm EP}_{s^n}$ and $\Delta^{\rm EP}_{h^n}$ processes collapse to zero as an effect of the non-linearity used (see Appendix~\ref{exp} for greater details): the EP error signals cannot be transmitted through saturated neurons, resulting in a RMSE of $\sim 10^{-1}$ for the network parameters - see Fig.~\ref{RMSE:conv} in Appendix~\ref{exp:conv}.

\begin{table}
    \caption{
    {\bfseries Above double rule}: 
    training results on MNIST with EP benchmarked against BPTT, in the energy-based and prototypical settings. "EB" and "P" respectively denote "energy-based" and "prototypical", "-$\#$h" stands for the number of hidden layers. We indicate over five trials the mean and standard deviation for the test error, the mean error in parenthesis for the train error.
    $T$ (resp. $K$) is the number of iterations in the first (resp. second) phase.
    {\bfseries Below double rule}:
    best training results on MNIST with EP reported in the literature.
    }
\begin{tabular}{@{}lccccccc@{}} \toprule
{}&\multicolumn{2}{c}{EP (error $\%$) }& \multicolumn{2}{c}{BPTT (error $\%$)}&{T}&{K}&{Epochs}\\
\cmidrule(r){2-3}\cmidrule(r){4-5}
{}& Test & Train & Test & Train & {} & {} & {} \\
\midrule

    {EB-1h} & $2.06 \pm 0.17$ & $(0.13)$ & $2.11 \pm 0.09$ & $(0.46) $  & 100 & 12  & 30\\ 
    {EB-2h} & $ 2.01 \pm 0.21$ & $(0.11) $ & $ 2.02 \pm 0.12 $ & $(0.29)$ & 500 & 40  & 50\\ 
    {P-1h} & $2.00 \pm 0.13$ & $(0.20) $ & $ 2.00 \pm 0.12 $& $(0.55) $ & 30 & 10  &  30 \\  
    {P-2h} & $ 1.95 \pm 0.10$ &  $(0.14)  $ & $ 2.09 \pm 0.12$ & $(0.37) $ & 100 & 20 &  50 \\  
    {P-3h} & $2.01 \pm 0.18$ & $(0.10) $ & $2.30 \pm 0.17$&  $(0.32) $ & 180 & 20  &  100\\      
    {P-conv} & $\mathbf{1.02 \pm 0.04}$ &  $(0.54) $ & $ 0.88 \pm 0.06 $ & $(0.12) $ & 200 & 10 &  40\\
\midrule
\midrule
\citep{Scellier+Bengio-frontiers2017} & $\sim 2.2 $ & ($\sim 0$)& - & - & 100 & 6 & 60 \\ 
\citep{OConnorICML2018} & $ 2.37$ & $(0.15)$ & - & - & 100 & 50 & 25 \\
\citep{pmlr-v89-o-connor19a} & $ 2.19$ & $ - $ & - & - & 4 & - & 50 \\
\bottomrule
\end{tabular}
    \label{table-results}
\end{table}

\section{Discussion}

Table~\ref{table-results} shows the accuracy results on MNIST
of several variations of our approach and of the literature. 
First, EP overall performs as well or practically as well as BPTT in terms of test accuracy in all situations.
Second, no degradation of accuracy is seen between using the prototypical (P) rather than the energy-based (EB) setting, although the prototypical setting requires  three to five times less time steps in the first phase (T).
Finally, the best EP result, $\sim 1 \%$ test error, is obtained with our convolutional architecture. This is also the best performance reported in the literature  on MNIST training with EP. BPTT achieves $~0.90\%$ test error using the same architecture. This slight degradation is due to saturated neurons which do no route error signals (as reported in the previous section). 
The prototypical situation allows using highly reduced number of time steps in the first phase than \citet{Scellier+Bengio-frontiers2017} and \citet{OConnorICML2018}. On the other hand, \citet{pmlr-v89-o-connor19a} manages to cut this number even more. This comes at the cost of using an extra network to learn proper initial states for the EP network, which is not needed in our approach.

Overall, our work broadens the scope of EP from its original formulation 
for biologically motivated real-time dynamics 
and sheds new light on its practical understanding.
We first extended EP to a discrete-time setting, which
reduces its computational cost and allows addressing situations closer to conventional machine learning.  Theorem~\ref{thm:main} demonstrated that the gradients provided by EP are strictly equal to the gradients computed with BPTT in specific conditions.
Our numerical experiments confirmed the theorem and showed that its range of applicability extends well beyond the original formulation of EP to prototypical neural networks widely used today. 
These results highlight that, in principle, EP can reach the same performance as BPTT on benchmark tasks (for RNN models with fixed input).
Layer-wise analysis of the gradients computed by EP and BPTT show that the deeper the layer, the more difficult it becomes to ensure the GDU property. 
On top of non-linearity effects, this is mainly due to the fact that the deeper the network, the longer it takes to reach equilibrium. 

While this may be a conundrum for current processors, 
it should not be an issue for alternative computing schemes.
Physics research is now looking at
neuromorphic computing approaches that leverage the transient dynamics of physical devices for computation \citep{torrejon2017neuromorphic,romera2018vowel,feldmann2019all}. 
In such systems, based on magnetism or optics, dynamical equations are solved directly by the physical circuits and components, in parallel and at speed much higher than processors.
On the other hand, in such systems, the nonlocality of backprop is a major concern  \citep{ambrogio2018equivalent}.
In this context, EP appears as a powerful approach as computing gradients only requires measuring the system at the end of each phase, and going backward in time is not needed. 
In a longer term, interfacing the algorithmics of EP with device physics could help cutting drastically the cost of inference and learning of conventional computers, and thereby address one of the biggest technological limitations of deep learning.

\newpage
\section*{Acknowledgments}

The authors would like to thank Joao Sacramento for feedback and discussions, as well as
NSERC, CIFAR, Samsung and Canada Research Chairs for funding.
Julie Grollier acknowledges funding from the European Research Council ERC under grant bioSPINspired 682955.

\bibliographystyle{abbrvnat}

\newpage

\appendix
\part*{Appendix}





\section{Proof of Theorem \ref{thm:main} - Step-by-Step Equivalence of EP and BPTT}
\label{sec:bptt-ep}

In this section, we prove that the first neural updates and synaptic updates performed in the second phase of EP are equal to the first gradients computed in BPTT (Theorem \ref{thm:main}).
In this section we choose a slightly different convention for the definition of the $\nabla^{\rm BPTT}_\theta(t)$ and $\Delta^{\rm EP}_\theta(t)$ processes, with an index shift. We explain in Appendix \ref{sec:index-shift} why this convention is in fact more natural.

\subsection{Backpropagation Through Time (BPTT)}
\label{sec:bptt}

Recall that we are considering an RNN (with fixed input $x$ and target $y$) whose dynamics $s_0, s_1, \ldots, s_T$ and loss ${\mathcal L}$ are defined by\footnote{Note that we choose here a different convention for the definition of $\theta_t$ compared to the definition of Section \ref{sec:background-bptt}. We motivate this index shift in Appendix \ref{sec:index-shift}.}
\begin{equation}
    \label{eq:definition-L}
    \forall t = 0, 1, \ldots T-1, \qquad s_{t+1} = F \left( x,s_t,\theta_t=\theta \right), \qquad {\mathcal L} = \ell \left( s_T, y \right).
\end{equation}
We denote the gradients computed by BPTT
\begin{align}
 \forall t = 0, 1, \ldots T, \qquad \nabla^{\rm BPTT}_s(t) & = \frac{\partial {\mathcal L}}{\partial s_{T-t}}, \\
 \forall t = 1, 2, \ldots T, \qquad \nabla^{\rm BPTT}_\theta(t) & = \frac{\partial {\mathcal L}}{\partial \theta_{T-t}}.
\end{align}
The gradients $\nabla^{\rm BPTT}_s(t)$ and $\nabla^{\rm BPTT}_\theta(t)$ are the `elementary gradients' (as illustrated in Fig.~\ref{fig:bptt}) computed as intermediary steps in BPTT in order to compute the `full gradient' $\frac{\partial {\mathcal L}}{\partial \theta}$.

\begin{prop}[Backpropagation Through Time]
\label{prop:bptt}
The gradients $\nabla^{\rm BPTT}_s(t)$ and $\nabla^{\rm BPTT}_\theta(t)$ can be computed using the recurrence relationship
\begin{align}
\nabla^{\rm BPTT}_s(0) & = \frac{\partial \ell}{\partial s} \left( s_T,y \right), \\
 \forall t=1,2,\ldots,T,    \qquad \nabla^{\rm BPTT}_s(t) & = \frac{\partial F}{\partial s} \left( x,s_{T-t},\theta \right)^\top \cdot \nabla^{\rm BPTT}_s(t-1), \label{eq:bptt-state} \\
  \forall t=1,2,\ldots,T,    \qquad \nabla^{\rm BPTT}_\theta(t) & = \frac{\partial F}{\partial \theta} \left( x,s_{T-t},\theta \right)^\top \cdot \nabla^{\rm BPTT}_s(t-1). \label{eq:bptt-param}
\end{align}
\end{prop}

\begin{proof}[Proof of Proposition \ref{prop:bptt}]
This is a direct application of the chain rule of differentiation, using the fact that $s_{t+1} = F \left( x,s_t,\theta \right)$
\end{proof}

\begin{cor}
\label{cor:bptt-steady-state}
In our specific setting with static input $x$, suppose that the network has reached the steady state $s_*$ after $T-K$ steps, i.e.
\begin{equation}
\label{eq:bptt-steady-condition}
s_{T-K} = s_{T-K+1} = \cdots = s_{T-1} = s_T = s_*.
\end{equation}
Then the first $K$ gradients of BPTT satisfy the recurrence relationship
\begin{align}
\label{eq:bptt-steady-init}
\nabla^{\rm BPTT}_s(0) & = \frac{\partial \ell}{\partial s} \left( s_*,y \right), \\
\label{eq:bptt-steady-state}
 \forall t=1,2,\ldots,K,    \qquad \nabla^{\rm BPTT}_s(t) & = \frac{\partial F}{\partial s} \left( x,s_*,\theta \right)^\top \cdot \nabla^{\rm BPTT}_s(t-1), \\
\label{eq:bptt-steady-param}
  \forall t=1,2,\ldots,K,    \qquad \nabla^{\rm BPTT}_\theta(t) & = \frac{\partial F}{\partial \theta} \left( x,s_*,\theta \right)^\top \cdot \nabla^{\rm BPTT}_s(t-1).
\end{align}
\end{cor}

\begin{proof}[Proof of Corollary \ref{cor:bptt-steady-state}]
    This is a straightforward consequence of Proposition \ref{prop:bptt}.
\end{proof}

\begin{figure}
\begin{center}
   \includegraphics[width=\textwidth]{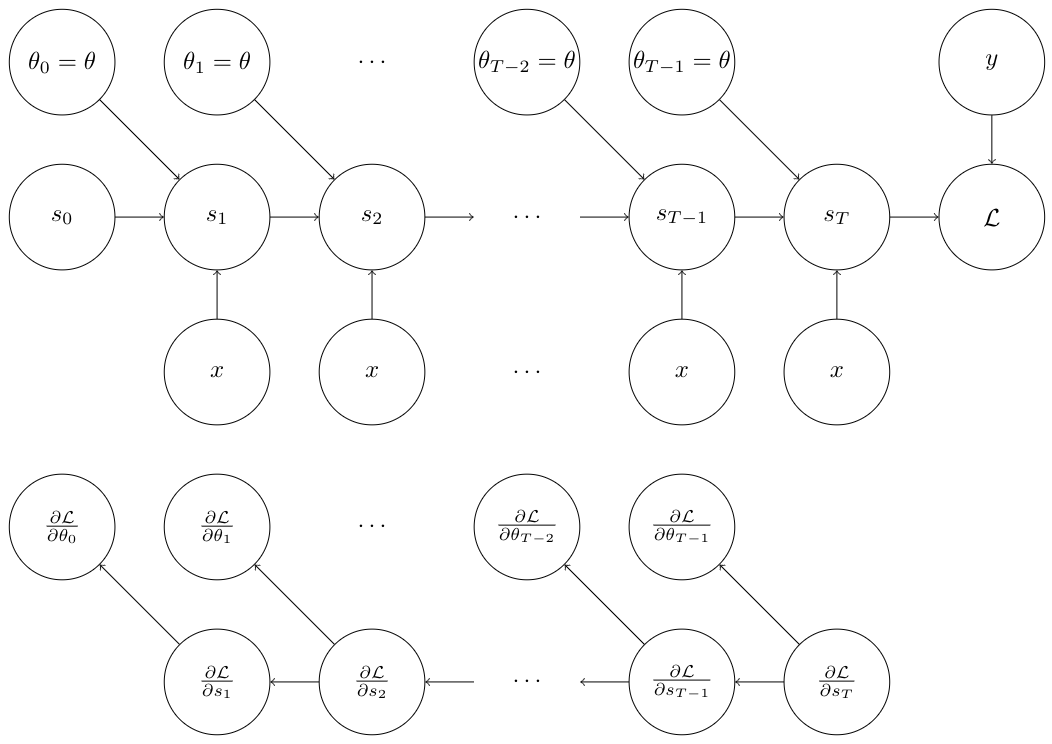}
\end{center}
  \caption{\textbf{Top.} Computational graph of an RNN with fixed input $x$ and target $y$, unfolded in time. As usual for RNNs, the parameters $\theta_0, \theta_1, \ldots, \theta_{T-1}$ at each time step share the same value $\theta$. The terminal state of the network is the steady state, i.e. $s_T = s_*$.
  \textbf{Bottom.} Backpropagation Through Time (BPTT) computes the gradients $\frac{\partial {\mathcal L}}{\partial s_T}, \frac{\partial {\mathcal L}}{\partial s_{T-1}}, \ldots, \frac{\partial {\mathcal L}}{\partial s_1}$ and $\frac{\partial {\mathcal L}}{\partial \theta_{T-1}}, \frac{\partial {\mathcal L}}{\partial \theta_{T-2}}, \ldots, \frac{\partial {\mathcal L}}{\partial \theta_0}$ as intermediary steps in order to compute the total gradient $\frac{\partial {\mathcal L}}{\partial \theta}$ as in Eq.~\ref{eq:total-gradient}.}
  \label{fig:bptt}
\end{figure}


\subsection{Equilibrium Propagation (EP) -- A Formulation with Arbitrary Transition Function $F$}
\label{sec:ep2}

In this section, we show (Lemma \ref{lemma:ep} below) that the neural updates $\Delta^{\rm EP}_s(t)$ and weight updates $\Delta^{\rm EP}_\theta(t)$ of EP satisfy a recurrence relation similar to the one for the gradients of BPTT (Proposition \ref{prop:bptt}, or more specifically Corollary \ref{cor:bptt-steady-state}).

In section \ref{sec:ep} we have presented EP in the setting where the transition function $F$ derives from a scalar function $\Phi$, i.e. with $F$ of the form $F(x,s,\theta) = \frac{\partial \Phi}{\partial s}(x,s,\theta)$.
This hypothesis is necessary to show equality of the updates of EP and the gradients of BPTT (Theorem \ref{thm:main}).
To better emphasize where this hypothesis is used, we first show an intermediary result (Lemma \ref{lemma:ep} below) which holds for arbitrary transition function $F$.

First we formulate EP for arbitrary transition function $F$, inspired by the ideas of \citet{scellier2018generalization}.
Recall that at the beginning of the second phase of EP the state of the network is the steady state $s^\beta_0 = s_*$ characterized by
\begin{equation}
	s_* = F \left( x,s_*,\theta \right),
\end{equation}
and that, given some value $\beta > 0$ of the hyperparameter $\beta$, the successive neural states $s^\beta_1, s^\beta_2, \ldots$ are defined and computed as follows:
\begin{equation}
	\forall t \geq 0, \qquad s_{t+1}^\beta = F \left( x,s_t^\beta,\theta \right) - \beta \; \frac{\partial \ell}{\partial s} \left( s_t^\beta,y \right).
	\label{eq:proof3-1}
\end{equation}
In this more general setting, we redefine the `neural updates' and `weight updates' as follows\footnote{Note the index shift in the definition of $\del{\theta}(\beta,t)$ compared to the definition of Eq.~\ref{eq:delta-ep}. We motivate this index shift in Appendix \ref{sec:index-shift}.}:
\begin{align}
    \label{eq:ep-state-update}
    \forall t \geq 0, \qquad \del{s}(\beta,t) & = \frac{1}{\beta} \left( s_{t+1}^\beta - s_t^\beta \right), \\
    \label{eq:ep-weight-update}
    \forall t \geq 1, \qquad \del{\theta}(\beta,t) & = \frac{1}{\beta} \frac{\partial F}{\partial \theta} \left( x,s_{t-1}^\beta,\theta \right)^\top
    \cdot \left( s_t^\beta - s_{t-1}^\beta \right).
\end{align}
In contrast, recall that in the gradient-based setting of section \ref{sec:ep} we had defined
\begin{equation}
    \label{eq:ep-weight-update-2}
    \del{\theta}(\beta,t) = \frac{1}{\beta} \left( \frac{\partial \Phi}{\partial \theta} \left( x,s_t^\beta,\theta \right) - \frac{\partial \Phi}{\partial \theta} \left( x,s_{t-1}^\beta,\theta \right) \right).
\end{equation}
When $F = \frac{\partial \Phi}{\partial s}$, the definitions of Eq.~\ref{eq:ep-weight-update} and Eq.~\ref{eq:ep-weight-update-2} are slightly different, but what matters is that both definitions coincide in the limit $\beta \to 0$.
Now that we have redefined $\del{s}(\beta,t)$ and $\del{\theta}(\beta,t)$ for general transition function $F$, we are ready to state our intermediary result.

\begin{lma}
	\label{lemma:ep}
    Let $\Delta^{\rm EP}_s(t) = \lim_{\beta \to 0} \Delta^{\rm EP}_s(\beta,t)$ and $\Delta^{\rm EP}_\theta(t) = \lim_{\beta \to 0} \Delta^{\rm EP}_\theta(\beta,t)$ be the neural and weight updates of EP in the limit $\beta \to 0$.
    They satisfy the recurrence relationship
	\begin{align}
	    \label{eq:ep-init}
	    \Delta^{\rm EP}_s(0)                   & = - \frac{\partial \ell}{\partial s}        \left( s_*,y \right), \\
	    \label{eq:ep-state}
	    \forall t \geq 0, \qquad \Delta^{\rm EP}_s(t+1)      & = \frac{\partial F}{\partial s}      \left( x,s_*,\theta \right)   \cdot \Delta^{\rm EP}_s(t), \\
	    \label{eq:ep-weight}
	    \forall t \geq 0, \qquad \Delta^{\rm EP}_\theta(t+1) & = \frac{\partial F}{\partial \theta} \left( x,s_*,\theta \right)^\top \cdot \Delta^{\rm EP}_s(t).
	\end{align}
\end{lma}

Note that the multiplicative matrix in Eq.~\ref{eq:ep-state} is the square matrix $\frac{\partial F}{\partial s}\left( x,s_*,\theta \right)$ whereas the one in Eq.~\ref{eq:bptt-steady-state} is its transpose $\frac{\partial F}{\partial s}\left( x,s_*,\theta \right)^\top$.
Because of that, the updates $\Delta^{\rm EP}_s(t)$ and $\Delta^{\rm EP}_\theta(t)$ of EP on the one hand, and the gradients $\nabla^{\rm BPTT}_s(t)$ and $\nabla^{\rm BPTT}_\theta(t)$ of BPTT on the other hand, satisfy different recurrence relationships in general.
Except when $\frac{\partial F}{\partial s}\left( x,s_*,\theta \right)$ is symmetric, i.e. when $F$ is of the form $F(x,s,\theta) = \frac{\partial \Phi}{\partial s}(x,s,\theta)$.

\begin{proof}[Proof of Lemma \ref{lemma:ep}]
	First of all, in the limit $\beta \to 0$, the weight update $\Delta^{\rm EP}_\theta(\beta,t)$ of Eq.~\ref{eq:ep-weight-update} rewrites
	\begin{equation}
	    \Delta^{\rm EP}_\theta(t) = 
        \frac{\partial F}{\partial \theta} \left( x,s_*,\theta \right)^\top
        \cdot \Delta^{\rm EP}_s(t).
	\end{equation}
	Hence Eq.~\ref{eq:ep-weight}.
	Now we prove Eq.~\ref{eq:ep-init}-\ref{eq:ep-state}.
	Note that the neural update $\Delta^{\rm EP}_s(\beta,t)$ of Eq.~\ref{eq:ep-state-update} rewrites
	\begin{equation}
		\Delta^{\rm EP}_s(t) = \left. \frac{\partial s_{t+1}^\beta}{\partial \beta} \right|_{\beta=0} - \left. \frac{\partial s_t^\beta}{\partial \beta} \right|_{\beta=0}.
	\end{equation}
	This is because for every $t \geq 0$ we have $s_t^\beta \to s_*$ as $\beta \to 0$ : starting from $s_0^0 = s_*$, if you set $\beta=0$ in Eq.~\ref{eq:proof3-1}, then $s_1^0 = s_2^0 = \ldots = s_*$.
	
	Differentiating Eq.~\ref{eq:proof3-1} with respect to $\beta$, we get
	\begin{equation}
		\forall t \geq 0, \qquad \frac{\partial s_{t+1}^\beta}{\partial \beta} = \frac{\partial F}{\partial s} \left( x,s_t^\beta,\theta \right) \cdot \frac{\partial s_t^\beta}{\partial \beta} - \frac{\partial \ell}{\partial s} \left( s_t^\beta,y \right) - \beta \frac{\partial^2 \ell}{\partial s^2} \left( s_t^\beta,y \right) \cdot \frac{\partial s_t^\beta}{\partial \beta}.
	\end{equation}
	Letting $\beta \to 0$, we have $s_t^\beta \to s_*$, so that
	\begin{equation}
		\label{eq:proof3-2}
		\forall t \geq 0, \qquad 
		\left. \frac{\partial s_{t+1}^\beta}{\partial \beta} \right|_{\beta=0} =
			\frac{\partial F}{\partial s} \left( x,s_*,\theta \right) \cdot
			\left. \frac{\partial s_t^\beta}{\partial \beta} \right|_{\beta=0}
			- \frac{\partial \ell}{\partial s} \left( s_*,y \right).
	\end{equation}
	Since at time $t=0$ the initial state of the network $s_0^\beta = s_*$ is independent of $\beta$, we have
	\begin{equation}
		\label{eq:proof3-4}
		\frac{\partial s_0^\beta}{\partial \beta} = 0.
	\end{equation}
	Using Eq.~\ref{eq:proof3-2} for $t=0$ and Eq.~\ref{eq:proof3-4}, we get the initial condition (Eq.~\ref{eq:ep-init})
	\begin{equation}
		\Delta^{\rm EP}_s(0)
		= \left. \frac{\partial s_1^\beta}{\partial \beta} \right|_{\beta=0} - \left. \frac{\partial s_0^\beta}{\partial \beta} \right|_{\beta=0}
		= - \frac{\partial \ell}{\partial s} \left( s_*,y \right).
	\end{equation}
	Moreover, if we take Eq.~\ref{eq:proof3-2} and subtract itself from it at time step $t-1$, we get
	\begin{equation}
		\Delta^{\rm EP}_s(t+1) = \frac{\partial F}{\partial s} \left( x,s_*,\theta \right) \cdot \Delta^{\rm EP}_s(t).
	\end{equation}
	Hence Eq.~\ref{eq:ep-state}.
	Hence the result.
\end{proof}

\subsection{Equivalence of EP and BPTT in the gradient-based Setting}

We recall Theorem \ref{thm:main}:

\theor*

\begin{proof}[Proof of Theorem \ref{thm:main}]
    This is a consequence of Corollary \ref{cor:bptt-steady-state} and Lemma \ref{lemma:ep}, using the fact that the Jacobian matrix of $F$ is the Hessian of $\Phi$, thus is symmetric:
	\begin{equation}
		\frac{\partial F}{\partial s} \left( x,s,\theta \right)^\top = \frac{\partial^2 \Phi}{\partial s^2} \left( x,s,\theta \right) = \frac{\partial F}{\partial s} \left( x,s,\theta \right).
	\end{equation}
\end{proof}

\section{Notations}

In this Appendix we motivate some of the distinctions that we make and the notations that we adopt. This includes:
\begin{enumerate}
    \item the distinction between the loss ${\mathcal L}^* = \ell \left( s_*,y \right)$ and the loss ${\mathcal L} = \ell \left( s_T,y \right)$, which may seem unnecessary since $T$ is chosen such that $s_T = s^*$,
    \item the distinction between the `primitive function' $\Phi(x,s,\theta)$ introduced in this paper and the `energy function' $E(x,s,\theta)$ used in \citet{Scellier+Bengio-frontiers2017,scellier2019equivalence},
    \item the index shift in the definition of the processes $\nab{\theta}(t)$ and $\del{\theta}(t)$.
\end{enumerate}

\subsection{Difference between ${\mathcal L}^*$ and ${\mathcal L}$}
\label{sec:loss-L-L*}

There is a difference between the loss at the steady state ${\mathcal L}^*$ and the loss after $T$ iterations ${\mathcal L}$.
To see why the functions ${\mathcal L}^*$ and ${\mathcal L}$ (as functions of $\theta$) are different, we have to come back to the definitions of $s_*$ and $s_T$.
Recall that
\begin{itemize}
    \item ${\mathcal L}^* = \ell \left( s_*,y \right)$ where $s_*$ is the steady state, i.e. characterized by $s_* = F \left( x,s_*,\theta \right)$,
    \item ${\mathcal L} = \ell \left( s_T,y \right)$ where $s_T$ is the state of the network after $T$ time steps, following the dynamics $s_0=0$ and $s_{t+1} = F \left( x,s_t,\theta \right)$.
\end{itemize}
For the current value of the parameter $\theta$, the hyperparameter $T$ is chosen such that $s_T = s_*$, i.e. such that the network reaches steady state after $T$ time steps. Thus, for this value of $\theta$ we have numerical equality ${\mathcal L}(\theta) = {\mathcal L}^*(\theta)$.
However, two functions that have the same value at a given point are not necessarily equal.
Similarly, two functions that have the same value at a given point don't necessarily have the same gradient at that point.
Here we are in the situation where
\begin{enumerate}
    \item the functions ${\mathcal L}$ and ${\mathcal L}^*$ (as functions of $\theta$) have the same value at the current value of $\theta$, i.e. ${\mathcal L}(\theta) = {\mathcal L}^*(\theta)$ numerically,
    \item the functions ${\mathcal L}$ and ${\mathcal L}^*$ (as functions of $\theta$) are analytically different, i.e. ${\mathcal L} \neq {\mathcal L}^*$.
\end{enumerate}
Since the functions ${\mathcal L}$ and ${\mathcal L}^*$ (as functions of $\theta$) are different, the gradients $\frac{\partial {\mathcal L}^*}{\partial \theta}$ and $\frac{\partial {\mathcal L}}{\partial \theta}$ are also different in general.

\subsection{Difference between the Primitive Function $\Phi$ and the Energy Function $E$}
\label{sec:primitive-energy}

Previous work on EP \citep{Scellier+Bengio-frontiers2017,scellier2019equivalence} has studied real-time dynamics of the form:
\begin{equation}
\label{eq:real-time-dynamics}
\frac{ds_t}{dt} = - \frac{\partial E}{\partial s} \left( x,s_t,\theta \right).
\end{equation}
In contrast, in this paper we study discrete-time dynamics of the form
\begin{equation}
\label{eq:discrete-time-dynamics}
s_{t+1} = \frac{\partial \Phi}{\partial s} \left( x,s_t,\theta \right).
\end{equation}
Why did we change the sign convention in the dynamics and why do we call $\Phi$ a `primitive function' rather than an `energy function'?
While it is useful to think of the primitive function $\Phi$ in the discrete-time setting as an equivalent of the energy function $E$ in the real-time setting, there is an important difference between $E$ and $\Phi$.
We argue next that, rather than an energy function, $\Phi$ is much better thought of as a primitive of the transition function $F$.
First we show how the two settings are related.

\paragraph{Casting real-time dynamics to discrete-time dynamics.}
The real-time dynamics of Eq.~(\ref{eq:real-time-dynamics}) can be cast to the discrete-time setting of Eq.~(\ref{eq:discrete-time-dynamics}) as follows.
The Euler scheme of Eq.~(\ref{eq:real-time-dynamics}) with discretization step $\epsilon$ reads: 
\begin{equation}
s_{t+1} = s_t - \epsilon \frac{\partial E}{\partial s} \left(x,s_t,\theta\right).
\end{equation}
This equation rewrites
\begin{equation}
s_{t+1} = \frac{\partial \Phi_\epsilon}{\partial s} \left( x,s_t,\theta \right),
\qquad \text{where} \qquad
\Phi_\epsilon(x,s,\theta) = \frac{1}{2} \| s \|^2 - \epsilon \; E(x,s,\theta).
\end{equation}
However, although the real-time dynamics can be mapped to the discrete-time setting, the discrete-time setting is more general.
The primitive function $\Phi$ cannot be interpreted in terms of an energy in general, as we argue next.

\paragraph{Why not keep the notation E and the name of `energy function' in the discrete-time framework?}
In the real-time setting, $s_t$ \textit{follows} the gradient of $E$, so that $E \left( s_t \right)$ decreases as time progresses until $s_t$ settles to a (local) minimum of $E$.
This property motivates the name of `energy function' for $E$ by analogy with physical systems whose dynamics settle down to low-energy configurations.
In contrast, in the discrete-time setting, $s_t$ \textit{is mapped} onto the gradient of $\Phi$ (at the point $s_t$).
In general, there is no guarantee that the discrete-time dynamics of Eq.~(\ref{eq:discrete-time-dynamics}) optimizes $\Phi$ and there is no guarantee that the dynamics of $s_t$ converges to an optimum of $\Phi$.
For this reason, there is no reason to call $\Phi$ an `energy function', since the intuition of optimizing an energy does not hold.

\paragraph{Why call $\Phi$ a `primitive function'?}

The name of `primitive function' for $\Phi$ is motivated by the fact that $\Phi$ is a primitive of the transition function $F$, whose property better captures the assumptions under which the theory of EP holds.
To see this, we first rewrite Eq.~(\ref{eq:discrete-time-dynamics}) in the form
\begin{equation}
    s_{t+1} = F \left( x,s_t,\theta \right),
\end{equation}
where $F$ is a transition function (in the state space) of the form
\begin{equation}
    \label{eq:gradient-field}
    F(x,s,\theta) = \frac{\partial \Phi}{\partial s} \left( x,s,\theta \right),
\end{equation}
with $\Phi(x,s,\theta)$ a scalar function.
For the theory of EP to hold (in particular Theorem \ref{thm:main}), the following two conditions must be satisfied (see Corollary \ref{cor:bptt-steady-state} and Lemma \ref{lemma:ep}):
\begin{enumerate}
    \item The steady state $s_*$ (at the end of the first phase and at the beginning of the second phase) must satisfy the condition
    \begin{equation}
    	s_* = F \left( x,s_*,\theta \right),
    \end{equation}
    \item the Jacobian of the transition function $F$ must be symmetric, i.e.
    \begin{equation}
        \label{eq:symmetric-vector-field}
    	\frac{\partial F}{\partial s}(x,s,\theta)^\top = \frac{\partial F}{\partial s}(x,s,\theta).
    \end{equation}
\end{enumerate}
The condition of Eq.~(\ref{eq:symmetric-vector-field}) is equivalent to the existence of a scalar function $\Phi(x,s,\theta)$ such that Eq.~(\ref{eq:gradient-field}) holds.
Going from Eq.~(\ref{eq:gradient-field}) to Eq.~(\ref{eq:symmetric-vector-field}) is straightforward: in this case the Jacobian of $F$ is the Hessian of $\Phi$, which is symmetric. Indeed $\frac{\partial F}{\partial s}(x,s,\theta) = \frac{\partial^2 \Phi}{\partial s^2}(x,s,\theta) = \frac{\partial F}{\partial s}(x,s,\theta)^\top$.
Going from Eq.~(\ref{eq:symmetric-vector-field}) to Eq.~(\ref{eq:gradient-field}) is also true -- though less obvious -- and is a consequence of Green's theorem.
\footnote{Another equivalent formulation is that the curl of $F$ is null, i.e. $\vec{\rm rot} \; \vec{F} = \vec{0}$.}
We say that $F$ derives from the scalar function $\Phi$, or that $\Phi$ is a primitive of $F$.
Hence the name of `primitive function' for $\Phi$.

\paragraph{Assumption of Convergence in the Discrete-Time Setting.}
As we said earlier, in the real-time setting the gradient dynamics of Eq.~\ref{eq:real-time-dynamics} guarantees convergence to a (local) minimum of $E$.
In contrast, in the discrete-time setting, no intrinsic property of $F$ or $\Phi$ a priori guarantees that the dynamics of Eq~\ref{eq:discrete-time-dynamics} settles to steady state.
This discussion is out of the scope of this work and we refer to \citet{scarselli2009graph} where sufficient (but not necessary) conditions are discussed to ensure convergence based on the contraction map theorem.

\subsection{Index Shift in the Definition of $\nab{\theta}(t)$ and $\del{\theta}(t)$}
\label{sec:index-shift}

The convention that we have chosen to define $\nab{\theta}(t)$ and $\del{\theta}(t)$ in Appendix \ref{sec:bptt-ep} may seem strange at first glance for two reasons:
\begin{itemize}
    \item the state update $\del{s}(t)$ is defined in terms of $s_t^\beta$ and $s_{t+1}^\beta$, whereas the weight update $\del{\theta}(t)$ is defined in terms of $s_{t-1}^\beta$ and $s_t^\beta$,
    \item at time $t=0$, the state gradient $\nab{s}(0)$ and the state update $\del{s}(0)$ are defined, but the weight gradient $\nab{\theta}(0)$ and the weight update $\del{\theta}(0)$ are not defined.
\end{itemize}
Here we explain why our choice of notations is in fact natural.
First, recall from Appendix \ref{sec:bptt} that we have defined the gradients of BPTT as
\begin{align}
    \forall t = 0,1,\ldots,T, \qquad \nab{s}(t) & = \frac{\partial {\mathcal L}}{\partial s_{T-t}}, \\
    \label{eq:aux2}
    \forall t = 1,2,\ldots,T, \qquad \nab{\theta}(t) & =
    \frac{\partial {\mathcal L}}{\partial \theta_{T-t}},
\end{align}
where
\begin{equation}
    \label{eq:aux1}
    \forall t = 0, 1, \ldots T-1, \qquad s_{t+1} = F \left( x,s_t,\theta_t=\theta \right), \qquad {\mathcal L} = \ell \left( s_T, y \right),
\end{equation}
and from Appendix \ref{sec:ep2} that we have defined the neural and weight updates of EP as
\begin{align}
    \forall t \geq 0, \qquad \del{s}(t) & = \lim_{\beta \to 0} \frac{1}{\beta} \left( s_{t+1}^\beta - s_t^\beta \right), \\
    \forall t \geq 1, \qquad \del{\theta}(t) & =
    \lim_{\beta \to 0} \frac{1}{\beta} \left( \frac{\partial \Phi}{\partial \theta} \left( x,s_t^\beta,\theta \right) - \frac{\partial \Phi}{\partial \theta} \left( x,s_{t-1}^\beta,\theta \right) \right),
\end{align}
where
\begin{equation}
	s_0^\beta = s_*, \qquad \forall t \geq 0, \quad s_{t+1}^\beta = F \left( x,s_t^\beta,\theta \right) - \beta \; \frac{\partial \ell}{\partial s} \left( s_t^\beta,y \right).
\end{equation}

\subsubsection{Index Shift}

Let us introduce
\begin{equation}
    \Phi^\beta(x,s,y,\theta) = \Phi(x,s,\theta) - \beta \; \ell(s,y),
\end{equation}
so that the dynamics in the second phase rewrites
\begin{equation}
    s_{t+1}^\beta = \frac{\partial \Phi^\beta}{\partial s} \left( x,s_t^\beta,y,\theta \right).
\end{equation}
It is then readily seen that the neural updates $\Delta_s^{\rm EP}$ and the weight updates $\Delta_\theta^{\rm EP}$ both rewrite in the form
\begin{align}
    \del{s}(0) & = \lim_{\beta \to 0} \frac{1}{\beta} \left( \frac{\partial \Phi^\beta}{\partial s} \left( x,s_0^\beta,y,\theta \right) - \frac{\partial \Phi}{\partial s} \left( x,s_0^\beta,\theta \right) \right), \label{eq:t=0} \\
    \forall t \geq 1, \qquad \del{s}(t) & = \lim_{\beta \to 0} \frac{1}{\beta} \left( \frac{\partial \Phi^\beta}{\partial s} \left( x,s_t^\beta,y,\theta \right) - \frac{\partial \Phi^\beta}{\partial s} \left( x,s_{t-1}^\beta,y,\theta \right) \right), \\
    \forall t \geq 1, \qquad \del{\theta}(t) & =
    \lim_{\beta \to 0} \frac{1}{\beta} \left( \frac{\partial \Phi^\beta}{\partial \theta} \left( x,s_t^\beta,y,\theta \right) - \frac{\partial \Phi^\beta}{\partial \theta} \left( x,s_{t-1}^\beta,y,\theta \right) \right).
\end{align}
Written in this form, we see a symmetry between $\del{s}(t)$ and $\del{\theta}(t)$ and there is no more index shift.

\subsubsection{Missing Weight Gradient $\nab{\theta}(0)$ and Weight Update $\del{\theta}(0)$}

We can naturally extend the definition of $\nab{\theta}(0)$ and $\del{\theta}(0)$ following Eq.~\ref{eq:aux2}.
In the setting studied in this paper, they both take the value $0$ because the cost function $\ell(s,y)$ does not depend on the parameter $\theta$.
But suppose now that $\ell$ depends on $\theta$, i.e. that $\ell$ is of the form $\ell(s,y,\theta)$.
Then the loss of Eq.~\ref{eq:aux1} takes the form ${\mathcal L} = \ell \left( s_T,y,\theta_T=\theta \right)$, so that:
\begin{equation}
    \nab{\theta}(0) = \frac{\partial {\mathcal L}}{\partial \theta_T} = \frac{\partial \ell}{\partial \theta} \left( s_T,y,\theta \right).
\end{equation}

As for the missing weight update $\del{\theta}(0)$, we follow the definition of Eq.~\ref{eq:t=0} and define:
\begin{equation}
    \del{\theta}(0) = \lim_{\beta \to 0} \frac{1}{\beta} \left( \frac{\partial \Phi^\beta}{\partial \theta} \left( x,s_0^\beta,y,\theta \right) - \frac{\partial \Phi}{\partial \theta} \left( x,s_0^\beta,\theta \right) \right) = - \frac{\partial \ell}{\partial \theta} \left( s_*,y,\theta \right).
\end{equation}
Since $s_T = s_*$ (the state at the end of the first phase is the state at the beginning of the second phase, and it is the steady state), we have $\del{\theta}(0) = -\nab{\theta}(0)$.

\newpage
\section{Experiments: demonstrating the GDU property}
\label{exp}

\subsection{Details on section \ref{sec:eb-prototypical-settings}}
\label{exp:toymodel}

\begin{figure}[ht!]
\begin{center}
   \fbox{\includegraphics[scale=0.25]{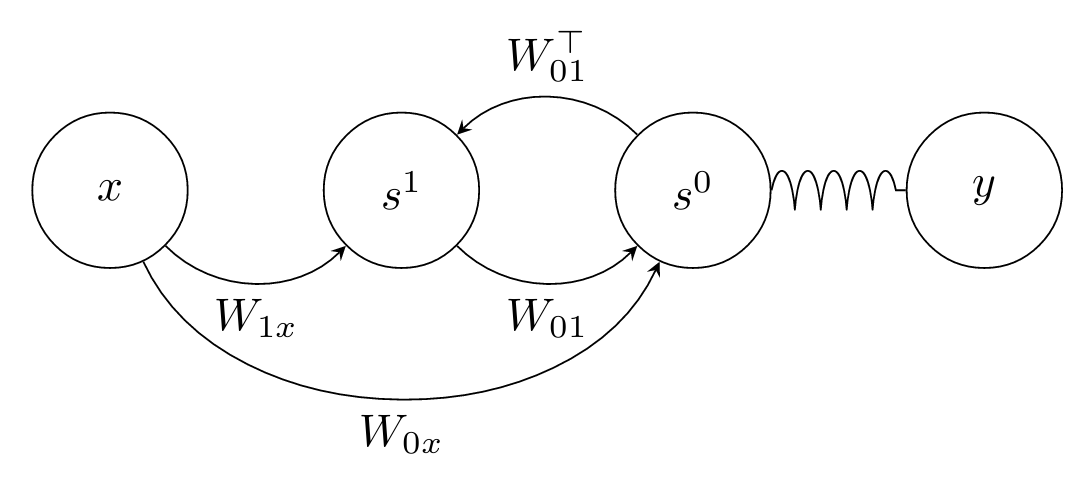}}
\end{center}
  \caption{Toy model architecture}
  \label{toymodelarchi}
\end{figure}
\paragraph{Equations.} The toy model is an architecture where input, hidden and output neurons are connected altogether, without lateral connections. Denoting input neurons as $x$, hidden neurons as $s_1$ and output neurons as $s_0$, the primitive function for this model reads:

\begin{align*}
\Phi \left( x, s^0, s^1 \right) &=  (1 - \epsilon)\frac{1}{2} \left( ||s^0||^{2} + ||s^1||^{2} \right) \\
+ \epsilon & \left(\sigma(s^0)\cdot W_{01}\cdot\sigma(s^1)  + \sigma(s^0)\cdot W_{0x}\cdot\sigma(x) + \sigma(s^1)\cdot W_{1x}\cdot\sigma(x) \right),
\end{align*}
where $\epsilon$ is a discretization parameter.
Furthermore the cost function $\ell$ is
\begin{equation}
    \ell \left( s^0,y \right) = \frac{1}{2} \| s^0 - y \|^2.
\end{equation}

As a reminder, we define the following convention for the dynamics of the second phase: $\forall t \in [0, K]: s^{n, \beta}_{t} = s^{n}_{t +T}$ where $T$ is the length of the first phase. The equations of motion read in the first phase read
\begin{align*}
\forall t \in [0, T]: 
\left\{
\begin{array}{ll}
    s^0_{t + 1} &= (1 - \epsilon)s^0_{t} + \epsilon\sigma'(s^0_{t}))\odot(W_{01}\cdot\sigma(s^1_{t}) + W_{0x}\cdot\sigma(x))\\
    s^1_{t + 1} &= (1 - \epsilon)s^1_{t} + \epsilon\sigma'(s^1_{t})\odot(W_{01}^\top\cdot \sigma(s^0_{t}) + W_{1x}\cdot\sigma(x)),
    \end{array}
\right.
\end{align*}
In the second phase
\begin{align}
\forall t \in [0, K]: 
\left\{
\begin{array}{ll}
    s^{0, \beta}_{t+1} &= (1 - \epsilon)s^{0, \beta}_{t} + \epsilon\sigma'(s^{0, \beta}_{t})\odot(W_{01}\cdot\sigma(s^{1, \beta}_{t}) + W_{0x}\cdot\sigma(x)) \\
        &+ \epsilon\beta(y - s^{0, \beta}_{t})\\
    s^{1, \beta}_{t + 1} &= (1 - \epsilon)s^{1, \beta}_{t} + \epsilon\sigma'(s^{1, \beta}_{t})\odot(W_{01}^\top\cdot \sigma(s^{0, \beta}_{t}) + W_{1x}\cdot\sigma(x)),
    \end{array}
\right.
\label{eq-toymodel-ep-cont}
\end{align}

where $y$ denotes the target. In this case and according to the definition Eq.~(\ref{eq:delta-ep}), the EP error processes for the parameters $\theta = \{W_{01}, W_{0x}, W_{1x}\}$ read:

\begin{equation*}
\forall t \in [0, K]:
    \left\{
\begin{array}{ll}
\Delta_{W_{01}}^{\rm EP}(t) &= \frac{1}{\beta}\left(\sigma(s^{0, \beta}_{t+1})\cdot\sigma(s^{1, \beta}_{t+1})^\top- \sigma(s^{0, \beta}_{t})\cdot\sigma(s^{1, \beta}_{t})^\top\right)\\
\Delta_{W_{0x}}^{\rm EP}(t) &= \frac{1}{\beta}\left(\sigma(s^{0, \beta}_{t+1})\cdot\sigma(x)^\top- \sigma(s^{0, \beta}_{t})\cdot\sigma(x)^\top\right)\\
\Delta_{W_{1x}}^{\rm EP}(t) &= \frac{1}{\beta}\left(\sigma(s^{1, \beta}_{t+1})\cdot\sigma(x)^\top- \sigma(s^{1, \beta}_{t})\cdot\sigma(x)^\top\right),
    \end{array}
\right.
\end{equation*}

\paragraph{Experiment: theorem demonstration on dummy data.} We took 5 output neurons, 50 hidden neurons and 10 visible neurons, using $\sigma(x) = \tanh(x)$. The experiment consists of the following: we define a dummy uniformly distributed random input $x \sim U[0, 1]$ (of size $1\times 10$) and a dummy random one-hot encoded target (of size $1\times 5$). We take $\epsilon = 0.08$ and perform the first phase for $T = 5000$ steps. Then, we perform on the one hand BPTT over $K= 80$ steps (to compute the gradients $\nabla^{\rm BPTT}$), on the other hand EP over $K = 80$ steps with $\beta = 0.01$ (to compute the neural updates $\Delta^{\rm EP}$) and compare the gradients and neural updates provided by the two algorithms. The resulting curves can be found in the main text (Fig.~\ref{fig:toy-model}).

\subsection{Details on subsection \ref{sec:depth}}

\subsubsection{Definition of the fully connected layered model in the energy-based setting}
\label{exp:def-fc-eb}
\begin{figure}[ht!]
\begin{center}
   \fbox{\includegraphics[scale=0.25]{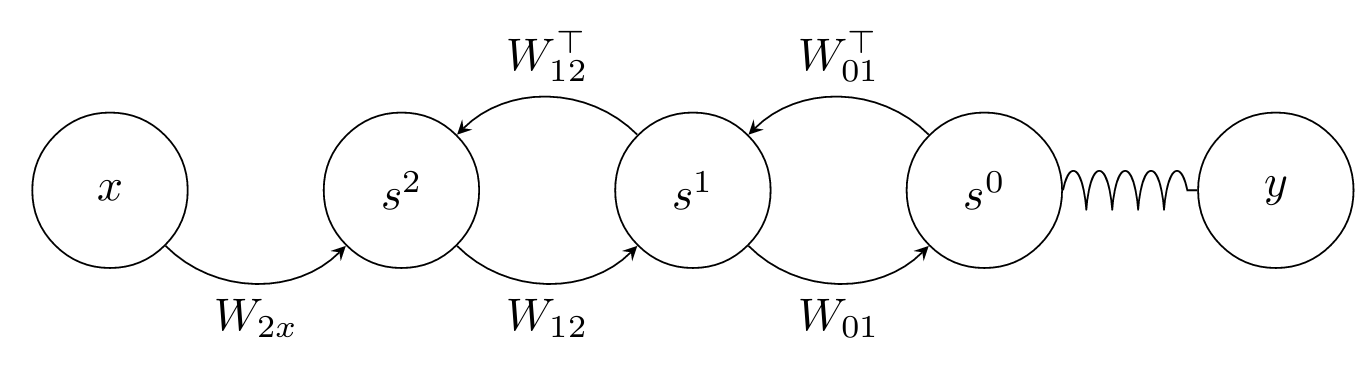}}
\end{center}
  \caption{Layered architecture}
\end{figure}
\paragraph{Equations.} The fully connected layered model is an architecture where the neurons are only connected between two consecutive layers. We denote neurons of the n-th layer as $s^{n}$ with $n\in [0, N - 1]$. Layers are labelled in a backward fashion: $n=0$ labels the output layer, $n=1$ the first hidden starting from the output layer, and $n = N - 1$ the visible layer so that there are $N-2$ hidden layers. As a reminder, we define the following convention for the dynamics of the second phase: $\forall t \in [0, K]: s^{n, \beta}_{t} = s^{n}_{t +T}$ where $T$ is the length of the first phase. The primitive function of this model is defined as:

\begin{equation}
\Phi \left( x, s^0, s^1, \ldots, s^N \right) =  \frac{1}{2}(1 - \epsilon) \left( \sum_{n = 1}^{N}||s^n||^{2} \right) + \epsilon \sum_{n = 0}^{N - 1} \sigma(s^n)\cdot W_{n n+1}\cdot\sigma(s^{n+1})
\end{equation}

so that the equations of motion read:

\begin{equation*}
\forall t \in [0, T] \\:
\left\{
\begin{array}{ll}
    s^0_{t + 1} &= (1 - \epsilon)s^0_{t} + \epsilon\sigma'(s^0_{t}))\odot W_{0 1}\cdot\sigma(s^{1}_{t})\\
    s^n_{t + 1} &= (1 - \epsilon)s^n_{t} + \epsilon\sigma'(s^n_{t}))\odot(W_{n n+1}\cdot\sigma(s^{n+1}_{t}) + W_{n -1 n}^\top\cdot\sigma(s^{n-1}_{t})) \qquad \forall n \in [1, N - 2]
\end{array}
\right.
\end{equation*}

\begin{equation}
\forall t \in [0, K]:
\left\{
\begin{array}{ll}
    s^{0, \beta}_{t + 1} &= (1 - \epsilon)s^{0, \beta}_{t} + \epsilon\sigma'(s^{0, \beta}_{t}))\odot W_{0 1}\cdot\sigma(s^{1, \beta}_{t}) + \beta\epsilon(y - s^{0, \beta}(t))\\
    s^{n, \beta}_{t + 1} &= (1 - \epsilon)s^{n, \beta}_{t} + \epsilon\sigma'(s^{n, \beta}_{t}))\odot(W_{n n+1}\cdot\sigma(s^{n+1, \beta}_{t}) + W_{n -1 n}^\top\cdot\sigma(s^{n - 1, \beta}_{t})) \\
    & \forall n \in [1, N - 2]
\end{array}
\right.
\label{eq-mnist-ep-cont}
\end{equation}

In this case and according to the definition Eq.~\ref{eq:delta-ep}, the EP error processes for the parameters $\theta = \{W_{n n+1}\}$ read:

\begin{equation*}
\forall t \in [0, K], \qquad \forall n \in [0, N-2]: \Delta_{W_{n n+1}}^{\rm EP}(t) = \frac{1}{\beta}\left(\sigma(s_{t+1}^{n, \beta})\cdot \sigma(s_{t+1}^{{n + 1, \beta}})^\top - \sigma(s_{t}^{n, \beta})\cdot \sigma(s_{t}^{{n + 1, \beta}})^\top  \right)
\end{equation*}

\paragraph{Experiment: theorem demonstration on MNIST.} For this experiment, we consider architectures of the kind 784-512-\dots-512-10 where we have 784 input neurons, 10 ouput neurons, and each hidden layer has 512 neurons, using $\sigma(x) = \tanh(x)$. The experiment consists of the following: we take a random MNIST sample (of size $1\times 784$) and the associated target (of size $1\times 10$). For a given $\epsilon$, we perform the first phase for $T = 2000$ steps. Then, we perform on the one hand BPTT over $K=$ steps (to compute the gradients $\nabla^{\rm BPTT}$), on the other hand EP over $K $ steps with a given $\beta$ (to compute the neural updates $\Delta^{\rm EP}$) and compare the gradients and neural updates provided by the two algorithms. Precise values of the hyperparameters $\epsilon$, T, K, beta are given in Tab.~\ref{table-hyp-th}.

\begin{figure}[H]
\begin{center}
   \includegraphics[scale=0.52]{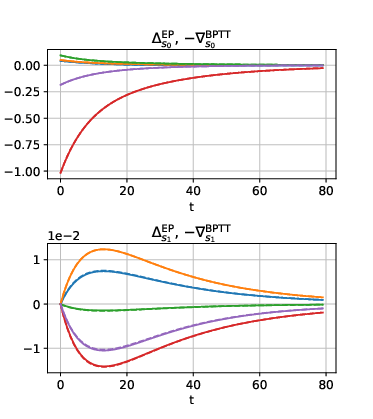}
      \hfill 
   \includegraphics[scale=0.52]{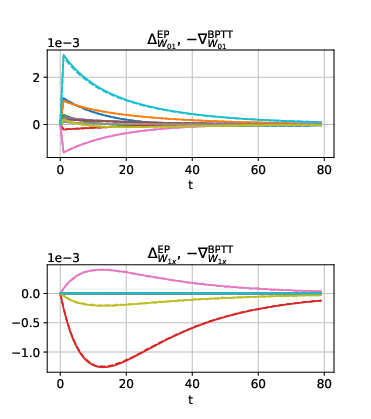}
\end{center}
\caption{Demonstrating the GDU property in the energy-based setting (as predicted by Theorem~\ref{thm:main}) with the fully connected layered architecture with one hidden layer on MNIST.}
\end{figure}

\begin{figure}[H]
\begin{center}
   \includegraphics[scale=0.52]{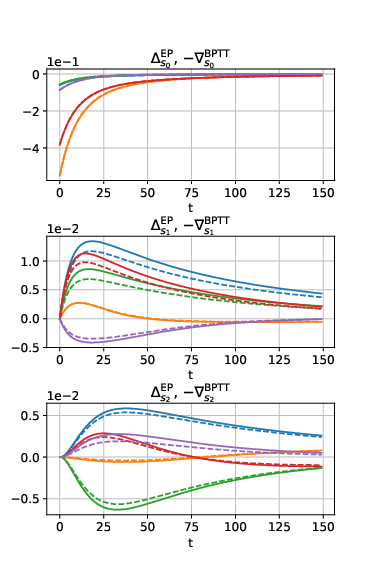}
      \hfill 
   \includegraphics[scale=0.52]{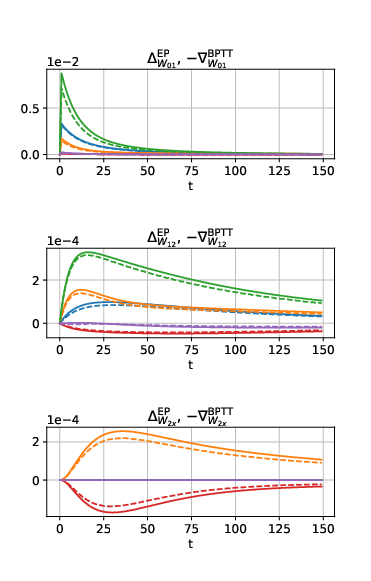}
\end{center}
\caption{Demonstrating the GDU property in the energy-based setting (as predicted by Theorem~\ref{thm:main}) with the fully connected layered architecture with two hidden layers on MNIST.}
\end{figure}

\begin{figure}[H]
\begin{center}
   \includegraphics[scale=0.52]{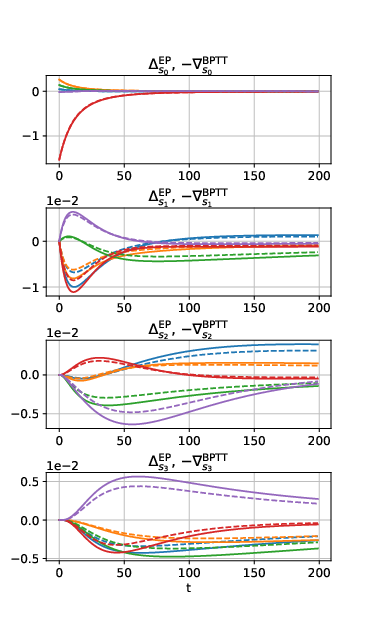}
      \hfill 
   \includegraphics[scale=0.52]{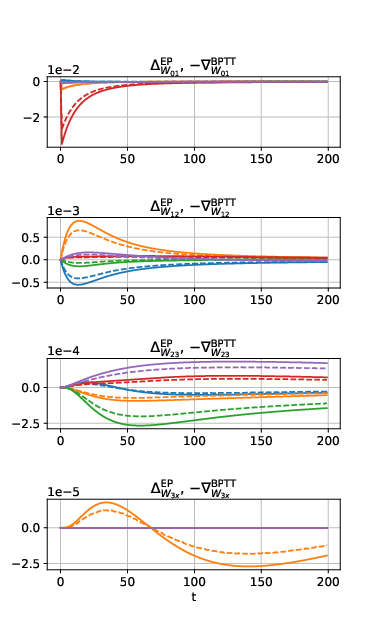}
\end{center}
\caption{Demonstrating the GDU property in the energy-based setting (as predicted by Theorem~\ref{thm:main}) with the fully connected layered architecture with three hidden layers on MNIST.}
\end{figure}

\newpage
\subsubsection{Fully connected layered architecture in the prototypical setting}
\label{exp:def-fc-p}

\paragraph{Equations.} The dynamics of the fully connected layered model are defined by the following set of equations:

\begin{equation*}
\forall t \in [0, T]:
\left\{
\begin{array}{ll}
    s^0_{t + 1} &= \sigma(W_{01}\cdot s^{1}_{t})\\
    s^n_{t + 1} &= \sigma(W_{n n+1}\cdot s^{n+1}_{t} + W_{n -1 n}^\top\cdot s^{n-1}_{t}) \qquad \forall n \in [1, N - 2]
\end{array}
\right.
\end{equation*}

\begin{equation*}
\forall t \in [0, K]:
\left\{
\begin{array}{ll}
    s^{0, \beta}_{t + 1} &= \sigma(W_{0 1}\cdot s^{1, \beta}_{t}) + \beta(y - s^{0, \beta}(t))\\
    s^{n, \beta}_{t + 1} &= \sigma(W_{n n+1}\cdot s^{n+1, \beta}_{t} + W_{n -1 n}^\top\cdot s^{n-1, \beta}_{t}) \qquad \forall n \in [1, N - 2],
\end{array}
\right.
\end{equation*}

where y denotes the target. Considering the function:

\begin{equation}
\Phi \left( x, s^0, s^1, \ldots, s^N \right) = \sum_{n = 0}^{N - 1} s^n\cdot W_{n n+1}\cdot s^{n+1}, 
\end{equation}

and ignoring the activation function, we have:

\begin{equation}
s_{t}^n \approx \frac{\partial \Phi}{\partial s^n}(x, s_{t-1}^0, \cdots, s_{t-1}^
{N-1})    
\end{equation}

so that in this case, we define the EP error processes for the parameters $\theta = \{W_{n n+1}\}$ as:

\begin{equation*}
\forall t \in [0, K], \qquad \forall n \in [0, N-2]: \Delta_{W_{n n+1}}^{EP}(t) =
\frac{1}{\beta}\left(s_{t+1}^{n, \beta}\cdot s_{t+1}^{{n+1, \beta}^\top} - s_{t}^{n, \beta}\cdot s_{T+t}^{{n+1, \beta}^\top}  \right)
\end{equation*}

\paragraph{Experiment: theorem demonstration on MNIST.} The experimental protocol is the exact same as the one used on the fully connected layered architecture in the energy-based setting, using the same activation function $\sigma(x)=\tanh(x)$. Precise values of the hyperparameters $\epsilon$, $T$, $K$, beta are given in Tab.~\ref{table-hyp-th}.

\begin{figure}[H]
\begin{center}
   \includegraphics[scale=0.52]{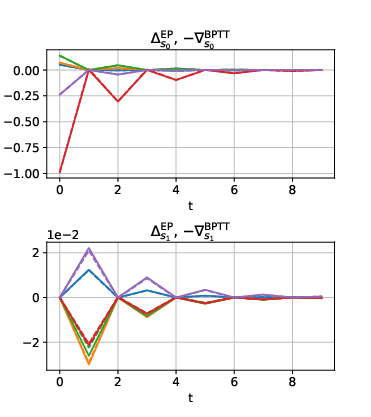}
      \hfill 
   \includegraphics[scale=0.52]{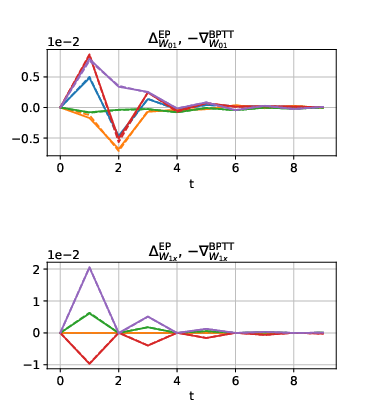}
\end{center}
\caption{Demonstrating the GDU property in the prototypical setting (as predicted by Theorem~\ref{thm:main}) with the fully connected layered architecture with one hidden layer on MNIST.}
\label{disc_1h_s}
\end{figure}

\begin{figure}[H]
\begin{center}
   \includegraphics[scale=0.52]{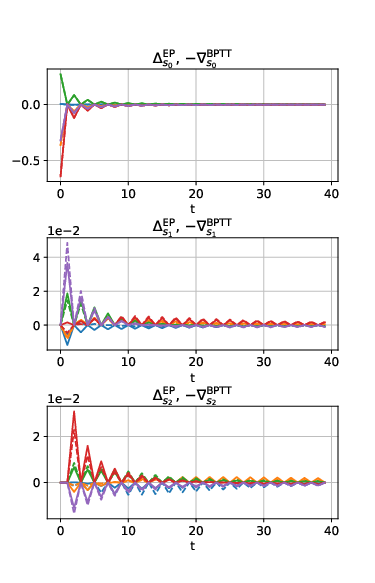}
      \hfill 
   \includegraphics[scale=0.52]{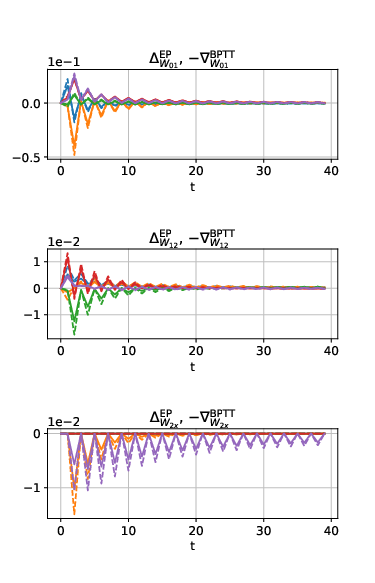}
\end{center}
\caption{Demonstrating the GDU property in the prototypical setting (as predicted by Theorem~\ref{thm:main}) with the fully connected layered architecture with two hidden layers on MNIST.}
\label{disc_2h_s}
\end{figure}

\begin{figure}[H]
\begin{center}
   \includegraphics[scale=0.5]{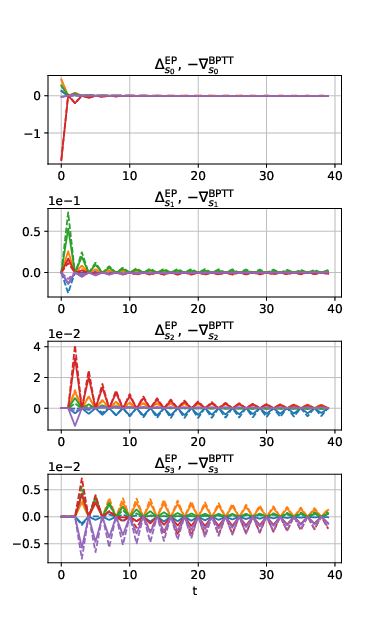}
      \hfill 
   \includegraphics[scale=0.5]{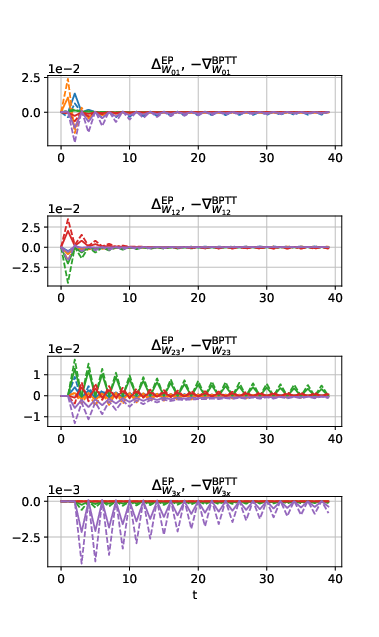}
\end{center}
\caption{Demonstrating the GDU property in the prototypical setting (as predicted by Theorem~\ref{thm:main}) with the fully connected layered architecture with three hidden layers on MNIST.}
\label{disc_3h_s}
\end{figure}

\subsubsection{Definition of the Relative Mean Square Error (RMSE)}
\label{exp:rmse}

We introduce a relative mean squared error (RMSE) between two continuous functions $f$ and $g$ in a given layer L as:

\begin{equation}
    {\rm RMSE(f,g)} = \left \langle\frac{\norm{f - g}_{2, K}}{{\rm max}(\norm{f}_{2, K}, \norm{g}_{2, K})}\right \rangle_{L}, 
\end{equation}

where $\norm{f}_{2, K}=\sqrt{\frac{1}{K}\int_{0}^{K}f^{2}(t)dt}$ and $\langle \cdot \rangle_{L}$ denotes an average over all the elements of layer L. For example, ${\rm RMSE}(\del{W_{01}}, -\nab{W_{01}})$ averages the squared distance between $\del{W_{01}}$ and $-\nab{W_{01}}$ averaged over all the elements of $W_{01}$. Also, instead of computing $\del{}$ and $\nab{}$ processes on a single sample presentation and bias the RMSE by the choice of this sample, $\del{}$ and $\nab{}$ processes have been averaged over a mini-batch of 20 samples before their distance in terms of RMSE was measured.

\subsubsection{Why are the $\nab{s}$ and $\del{s}$ processes saw teeth shaped in the prototypical setting ?}
\label{subsec-sawteeth}

In the prototypical setting, in the case of a layered architecture (without lateral and skip-layer connections),
the $\nab{}$ and $\del{}$ processes are saw teeth shaped, i.e. they take the value zero every other time step
(as seen per Fig.~\ref{th-conv}, Fig.~\ref{disc_1h_s}, Fig.~\ref{disc_2h_s} and Fig.~\ref{disc_3h_s}).
We provide an explanation for this phenomenon both from the point of view of BPTT and from the point of view of EP.
Fig.~\ref{saw_teeth} illustrates this phenomenon in the case of a network with two layers: one output layer $s^0$ and one hidden layer $s^1$.

\begin{itemize}
    \item {\bfseries Point of view of BPTT.} In the forward-time pass (first phase), $s_{t+1}^0$ is determined by $s_t^1$, while $s_{t+1}^1$ is determined by $s_t^0$.
    This gives rise to a zig-zag shaped connectivity pattern in the computational graph of the the network unrolled in time (Fig.~\ref{saw_teeth}).
    In particular, the gray nodes of Fig.~\ref{saw_teeth} are not involved in the computation of the loss ${\mathcal L}$, i.e. their gradients are equal to zero.
    In other words $\nab{s^1}(0) = 0$, $\nab{s^0}(1) = 0$, $\nab{s^1}(2) = 0$, etc.
    \item {\bfseries Point of view of EP.} At the beginning of the second phase (at time step $t=0$), the network is at the steady state ; in particular $s_0^{1,\beta} = s_*^1$.
    At time step $t=1$, only the output layer $s^0$ is influenced by $y$ ; the hidden layer $s^1$ is still at the steady state, i.e. $s_1^{1,\beta} = s_*^1$. From $s_0^{1,\beta} = s_1^{1,\beta}$, it follows that $s_1^{0,\beta} = s_2^{0,\beta}$. In turn, from $s_1^{0,\beta} = s_2^{0,\beta}$ it follows that $s_2^{1,\beta} = s_3^{1,\beta}$. Etc. In other words $\Delta_{s^1}^{\rm EP}(0) = 0$, $\Delta_{s^0}^{\rm EP}(1) = 0$, $\Delta_{s^1}^{\rm EP}(2) = 0$, etc.
\end{itemize}

The above argument can be generalized to an arbitrary number of layers.
In this case we group the layers of even index (resp. odd index) together.
We call $e_t = \left( s^0_t, s^2_t, s^4_t, \ldots \right)$ and $o_t = \left( s^1_t, s^3,t, s^5_t, \ldots \right)$.
The crucial property is that $o_{t+1}$ (resp. $e_{t+1}$) is determined by $e_t$ (resp. $o_t$).

\begin{figure}[ht!]
\begin{center}
   \fbox{\includegraphics[width=\textwidth]{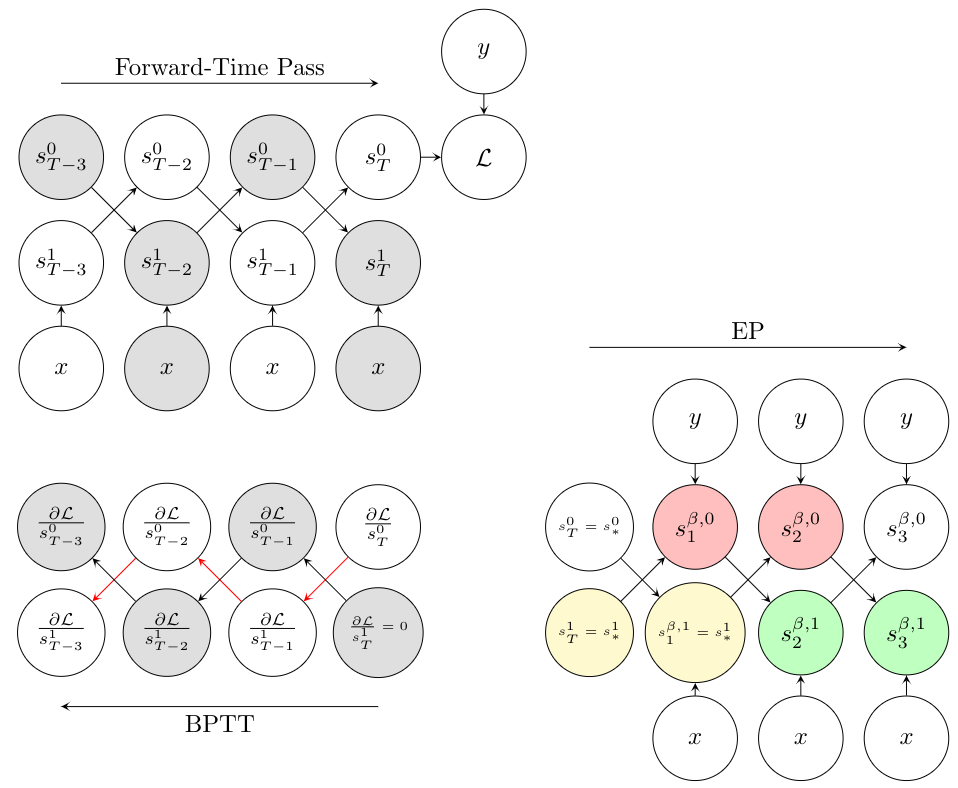}}
\end{center}
\caption{Explanation of the saw teeth shape of the $\nab{s}$ and $\del{s}$ processes in the prototypical setting (layered architecture without lateral or skip-layer connections). {\bfseries Forward-time pass (top left)}: gray nodes in the computational graph indicate nodes that are not involved in the computation of the loss ${\mathcal L}$. {\bfseries BPTT (bottom left)}: red arrows indicate the differentiation path through the output units $s^0$.
The gradients in the gray nodes are equal to $0$.
{\bfseries EP (bottom right)}: nodes of the same color have the same value.
}
\label{saw_teeth}
\end{figure}

In contrast, the saw teeth shaped curves are not observed in the energy based setting.
This is due to the different topology of the computational graph in this setting.
In the energy-based setting, the assumptions under which we have shown the saw teeth shape are not satisfied since neurons are subject to leakage, e.g. $s_{t+1}^1$ depends not just on $s_{t}^0$ but also on $s_{t}^1$.
Therefore the reasoning developed above no longer holds.

\newpage
\section{Convolutional model (subsection \ref{subsec-conv}) }
\label{exp:conv}

\begin{figure}[ht!]
\begin{center}
   \fbox{\includegraphics[scale=0.28]{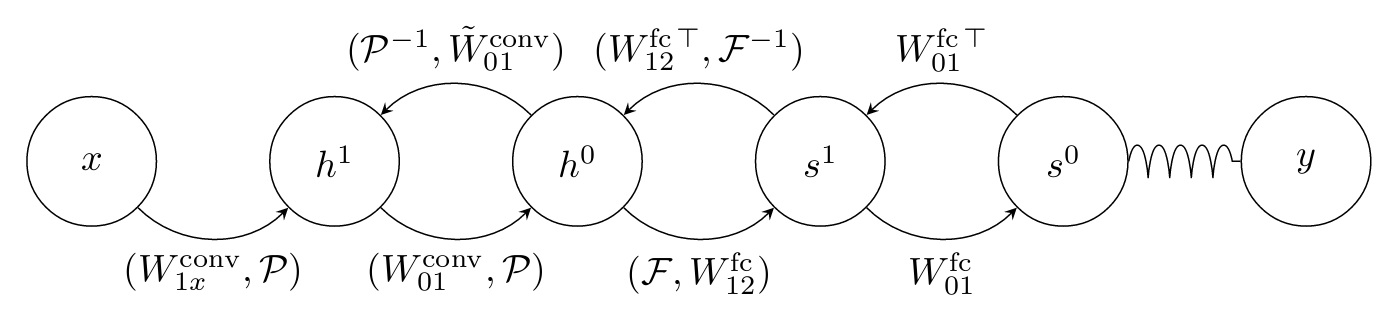}}
\end{center}
  \caption{Convolutional architecture. Summary of the operations, notations and conventions adopted in this section.}
\end{figure}
\paragraph{Definition of the operations.}
In this section, we define the following operations:

\begin{itemize}

\item the \emph{convolution} of a filter W of size F with $C_{\rm out}$ output channels and $C_{\rm in}$ input channels by a vector X as:

\begin{equation}
    (W*X)_{c_{\rm out},i,j}:=\sum_{c_{\rm int} = 1}^{C_{\rm in}}\sum_{r,s=1}^{F}W_{c_{\rm out}, c_{\rm in}, r, s}X_{c_{\rm in }, i+r-1, j+s-1} \mbox{.}
    \label{defconv},
\end{equation}

\item the associated \emph{transpose convolution} is defined as the convolution of kernel $\tilde{W}$ (also called "flipped kernel"):

\begin{equation}
    \tilde{W}_{c_{in}, c_{out, r, s}} = W_{c_{out}, c_{in}, F-r+1, F-s+1},
\end{equation}

with an input padded with $\tilde{P}=F-1 -P$ where P denotes the padding applied in the forward convolution: in this way transpose convolution recovers the original input size before convolution. Whenever $\tilde{W}$ is applied on a vector, we shall implicitly assume this padding.

\item We define the \emph{general dot product} between two vectors $X^1$ and $X^2$ as:

\begin{equation}
    X^{1} \bullet X^{2} = \sum_{c_{\rm in} = 1}^{C_{\rm in}}\sum_{i, j = 1}^{d} X^{1}_{c_{\rm in}, i, j}X^{2}_{c_{\rm in}, i, j}
    \label{defdot}.
\end{equation}

\item We define the \emph{pooling} operation with filter size F and stride F as:

\begin{equation}
    \mathcal{P}(X; F)_{c, i, j}=\max_{r,s \in [0, F-1]} \left \{X_{c, F(i-1)+1+r, F(j-1)+1+s} \right \}.
\end{equation}

We also introduce the relative indices within a pooling zone for which the maximum is reached as:

\begin{equation}
    {\rm ind}(X; F)_{c, i, j}=\argmax_{r,s \in [0, F-1]} \left \{X_{c, F(i-1)+1+r, F(j-1)+1+s} \right \} = (r^*(X,i), s^*(X,j)).
\end{equation}

\item We define the \emph{inverse pooling} operation as:

\begin{equation}
    \mathcal{P}^{-1}(Y, {\rm ind}(X))_{c,p,q}    
    =\left\{
\begin{array}{ll}
Y_{c, \ceil{p/F}, \ceil{q/F}}\mbox{if    } &p=F(\ceil{p/F}-1) + 1 + r^*(X,\ceil{p/F}),\\ 
&q=F(\ceil{q/F}-1) + 1 + s^*(X,\ceil{q/F}) \\
0 & \mbox{otherwise}
\end{array}
\right.
\label{definvpool}
\end{equation}

In layman terms, the inverse pooling operation applied to a vector $Y$ given the indices of another vector $X$ up-samples $Y$ to a vector of the same size of $X$ with the elements of Y located at the maximal elements of $X$ within each pooling zone, and zero elsewhere.

Note that Eq.~(\ref{definvpool}) can be written more conveniently as:

\begin{equation}
\mathcal{P}^{-1}(Y, {\rm ind}(X))_{c,p,q} = \sum_{i,j}Y_{c,i,j} \cdot \delta_{p, F(i-1) + 1 + r^*(X,i)} \cdot \delta_{q, F(j-1) + 1 + s^*(X,j)}.  
\label{invpool}
\end{equation}

\item The \emph{flattening} operation which maps a vector X into its flattened shape, i.e. $\mathcal{F}: C^{\rm in}\times D \times D \rightarrow 1\times C^{\rm in}D^2$. We denote its inverse operation, i.e. the \emph{inverse flattening operation} as $\mathcal{F}^{-1}$.

\end{itemize}

\paragraph{Equations.} The model is a layered architecture composed of a fully connected part and a convolutional part. We therefore distinguish between the flat layers (i.e. those of the fully connected part) and the convolutional layers (i.e. those of the convolutional part). We denote $N_{\rm fc}$ and $N_{\rm conv}$ the number of flat layers and of convolutional layers respectively. 

As previously, layers are labelled in a backward fashion: $s^0$ labels the output layer, $s^1$ the first hidden starting from the output layer (i.e. the first flat layer), and $s^{N_{\rm fc} - 1}$ the last flat layer. Fully connected layers are bi-dimensional\footnote{Three-dimensional in practice, considering the mini-batch dimension.}, i.e. $s_{i,j}$ where i and j label one pixel. 

The layer $h^{0}$ denotes the first convolutional layer that is being flattened before being fed to the classifier part. From there on, $h^{1}$ denotes the second convolutional layer, $h^{N_{\rm conv} - 1}$ the last convolutional layer and $h^{N_{\rm conv}}$ labels the visible layer. Convolutional layers are three-dimensional \footnote{Four-dimensional in pratice, considering the mini-batch dimension.}, i.e. $s_{c,i,j}$ where c labels a channel, i and j label one pixel of this channel.

A convolutional layer $h_{n}$ is deduced from an \emph{upstream} convolutional layer $h_{n-1}$ by the composition of a convolution and a pooling operation, which we shall respectively denote by $*$ and $\mathcal{P}$. Conversely, a convolutional layer $h_{n}$ is deduced from a \emph{downstream} convolutional layer $h_{n+1}$ by the composition of a unpooling operation and of a transpose convolution. We note $W^{\rm fc}$ and $W^{\rm conv}$ the fully connected weights and the convolutional filters respectively, so that $W^{\rm fc}$ is a two-order tensor and $W^{\rm conv}$ is a four order tensor, i.e. $W^{\rm conv}_{c_{out}, c_{in}, i, j}$ is the element $(i,j)$ of the feature map connecting the input channel $c_{in}$ to the output channel $c_{out}$. We denote the filter size by F. We keep the same notation $x$ for the input data.

With this set of notations, the equations in the fully connected layers read in the first phase:

\begin{equation*}
\forall t \in [0, T]:
\left\{
\begin{array}{ll}
    s^0_{t + 1} &= \sigma \left (W^{\rm fc}_{01}\cdot s^{1}_{t}\right) \qquad \mbox{(output layer)}\\
    s^n_{t + 1} &= \sigma \left(W^{\rm fc}_{n n+1}\cdot s^{n+1}_{t} + W^{{\rm fc}^\top}_{n -1 n}\cdot s^{n-1}_{t} \right) \qquad \forall n \in [1, N_{\rm fc} - 2]\\
    s^{N_{\rm fc} - 1}_{t+1} &= \sigma\left(W^{\rm fc}_{N_{\rm fc} - 1, N_{\rm fc}}\cdot \mathcal{F}(h_{t}^{0}) + W^{{\rm fc}^\top}_{N_{\rm fc} - 2, N_{\rm fc} - 1}\cdot s_{t}^{N_{\rm fc} - 2}\right) \qquad \mbox{(last fully connected layer)}
\end{array}, 
\right.
\end{equation*}

and in the second phase: 

\begin{equation*}
\forall t \in [0, T]:
\left\{
\begin{array}{ll}
    s^0_{t + 1} &= \sigma \left (W^{\rm fc}_{01}\cdot s^{1}_{t}\right) + \beta(y - s^0)\qquad \mbox{(nudged output layer)}\\
    s^n_{t + 1} &= \sigma \left(W^{\rm fc}_{n n+1}\cdot s^{n+1}_{t} + W^{{\rm fc}^\top}_{n -1 n}\cdot s^{n-1}_{t} \right) \qquad \forall n \in [1, N_{\rm fc} - 2]\\
    s^{N_{\rm fc} - 1}_{t+1} &= \sigma\left(W^{\rm fc}_{N_{\rm fc} - 1, N_{\rm fc}}\cdot \mathcal{F}(h_{t}^{0}) + W^{{\rm fc}^\top}_{N_{\rm fc} - 2, N_{\rm fc} - 1}\cdot s_{t}^{N_{\rm fc} - 2}\right) \qquad \mbox{(last fully connected layer)}
\end{array}, 
\right.
\end{equation*}

where y denotes the target. Conversely, convolutional layers read the following set of equations at all time: 

\begin{equation*}
\forall t \in [0, T]:
\left\{
\begin{array}{ll}
    h^{0}_{t+1} &= \sigma \left(\mathcal{P}\left(W^{\rm conv}_{01}*h^1_t\right) + \mathcal{F}^{-1} \left(W^{{\rm fc}^\top}_{N_{\rm fc} - 1, N_{\rm fc}}\cdot s_{t}^{N_{\rm fc} -1}\right)\right) \qquad \mbox{(first convolutional layer)} \\
    h^{n}_{t+1} &= \sigma \left( \mathcal{P}\left(W^{\rm conv}_{n, n+1}*h^{n+1}_t\right) + \tilde{W}_{n-1, n}^{\rm conv}*\mathcal{P}^{-1}\left(h^{n-1}_{t}, {\rm ind}(W^{\rm conv}_{n-1, n}*h^{n}_{t-1})\right)\right) \forall n \in [1, N_{\rm conv} - 1]
\end{array}, 
\right.
\end{equation*}

where by convention $h^{N_{\rm conv}}=x$. From here on, we shall omit the second argument of inverse pooling $\mathcal{P}^{-1}$ - i.e. the locations of the maximal neuron values before applying pooling - to improve readability of the equations and proofs. Considering the function:

\begin{align*}
    \Phi(x, s_{0}, \cdots, s_{N_{\rm fc} - 1}, h_{0}, \cdots, h_{N_{\rm fc} - 1}) &= 
    \sum_{n=0}^{N_{\rm fc - 1}} s^{n \top}\cdot W_{n, n+1}^{\rm fc}\cdot s^{n+1}
    + s^{N_{\rm fc} - 1}\cdot W^{\rm fc}_{N_{\rm fc} - 1, N_{\rm fc}}\cdot \mathcal{F}(h_{t}^0)\\
    &+\sum_{n = 1}^{N_{\rm conv} - 1} h^{n}\bullet\mathcal{P}\left(W_{n, n+1}^{\rm conv}*h^{n+1}\right), 
\end{align*}

and ignoring the activation function, we have:

\begin{equation}
   \left\{
\begin{array}{ll}
\forall n \in [0, N_{\rm fc}-1]: \qquad s_t^n \approx  \frac{\partial \Phi}{\partial s^n}(x, s_{0}, \cdots, s_{N_{\rm fc} - 1}, h_{0}, \cdots, h_{N_{\rm fc} - 1}) \\
\forall n \in [0, N_{\rm conv}-1]: \qquad h_t^n \approx  \frac{\partial \Phi}{\partial h^n}(x, s_{0}, \cdots, s_{N_{\rm fc} - 1}, h_{0}, \cdots, h_{N_{\rm fc} - 1}) \\
\end{array}, 
\right. 
\label{approxdyn}
\end{equation}

so that in this case, we  define the EP error processes for the parameters $\theta = \{ W^{\rm fc}_{nn+1}, W^{\rm conv}_{nn+1} \}$ as:

\begin{equation*}
\forall t \in [0, K], \forall n \in [0, N_{\rm fc}-2]: \qquad \Delta_{W_{n n+1}^{\rm fc}}^{\rm EP}(t)  =
\frac{1}{\beta}\left(s_{T+t+1}^n\cdot s_{T+t+1}^{{n+1}^\top} - s_{T+t}^n\cdot s_{T+t}^{{n+1}^\top}  \right) 
\end{equation*}

\begin{equation*}
\forall t \in [0, K]: \qquad \Delta_{W_{N_{\rm fc} - 1, N_{\rm fc}}^{\rm fc}}^{\rm EP}(t)  =
\frac{1}{\beta}\left(s_{T+t+1}^{N_{\rm fc} - 1}\cdot \mathcal{F}\left(h^0_{T+t+1}\right)^\top - s_{T+t}^{N_{\rm fc} - 1}\cdot \mathcal{F}\left(h^0_{T+t}\right)^\top \right)
\end{equation*}

\begin{equation}
\forall t \in [0, K], \forall n \in [0, N_{\rm conv}-2]:\qquad \Delta_{W_{n n+1}^{\rm conv}}^{\rm EP}(t)  =  \frac{1}{\beta} \left(\mathcal{P}^{-1}(h^{n}_{T+t+1})*h^{n+1}_{T+t+1} - \mathcal{P}^{-1}(h^{n}_{T+t})*h^{n+1}_{T+t} \right)  
\label{deltaconv}
\end{equation}

To further justify Eq.~(\ref{approxdyn}) and Eq.~(\ref{deltaconv}), we state and prove the following lemma. 

\begin{lma}
\label{lma:conv}
Taking:
\begin{equation*}
\Phi = Y \bullet \mathcal{P}\left(W * X \right), 
\end{equation*}
and denoting $Z=W*X$, we have:

\begin{equation}
\frac{\partial \Phi}{\partial Z} = \mathcal{P}^{-1}\left (Y\right)
\label{lma1}
\end{equation}

\begin{equation}
\frac{\partial \Phi}{\partial X} = \tilde{W}*\mathcal{P}^{-1}\left (Y\right)  
\label{lma2}
\end{equation}

\begin{equation}
\frac{\partial \Phi}{\partial W} = \mathcal{P}^{-1}\left (Y\right)*X    
\label{lma3}
\end{equation}

\begin{equation}
\frac{\partial \Phi}{\partial Y} = \mathcal{P}\left (W*X\right)
\label{lma4}
\end{equation}

\end{lma}

\begin{proof}[Proof of Lemma \ref{lma:conv}]
Let us prove Eq.~(\ref{lma1}). We have:

\begin{align*}
    \frac{\partial \Phi}{\partial Z_{c, x, y}}&=\sum_{c', i, j}Y_{c', i, j}\frac{\partial \mathcal{P}(Z)_{c', i, j}}{\partial Z_{c, x, y}}\\
    &=\sum_{c', i, j}Y_{c', i, j}\frac{\partial Z_{c', F(i-1)+1+r^*(i),F(j-1)+1+s^*(j)}}{\partial Z_{c, x, y}}\\
    &=\sum_{i,j}Y_{c,i,j}\delta_{x,F(i-1)+1+r^*(i)}\delta_{y,F(j-1)+1+s^*(j)}\\
    &=\mathcal{P}^{-1}(Y)_{c,x,y},
\end{align*}
where we used Eq.~(\ref{invpool}) at the last step.

We can now proceed to proving Eq.~(\ref{lma2}). We have:

\begin{align*}
    \frac{\partial \Phi}{\partial X_{c,p,q}}&=\sum_{c',x,y}\frac{\partial \Phi}{\partial Z_{c',x,y}}\cdot\frac{\partial Z_{c',x,y}}{\partial X_{c,p,q}}\\
    &=\sum_{c',x,y}\mathcal{P}^{-1}(Y)_{c',x,y}\cdot\frac{\partial}{\partial X_{c,p,q}}\left (\sum_{c'', r, s} W_{c', c'', r, s}X_{c'', x+r-1, y+s-1}\right)\\
    &=\sum_{c',x,y}\sum_{r,s}\mathcal{P}^{-1}(Y)_{c',x,y} W_{c',c,r,s}\delta_{p,x+r-1}\delta_{q, y + s - 1}\\
    &=\sum_{c',r,s}W_{c',c,r,s}\mathcal{P}^{-1}(Y)_{c',p-(r-1),q-(s-1)}.
\end{align*}

Using the flipped kernel $\tilde{W}$ and performing the change of variable $r\gets F - r + 1$ and $s \gets F - s + 1$, we obtain:

\begin{equation}
\frac{\partial \Phi}{\partial X_{c,p,q}} = \sum_{c',r, s} \tilde{W}_{c, c', r, s}\cdot \mathcal{P}^{-1}(Y)_{c',p+r-F, q+s-F}.
\label{proof-lma2}
\end{equation}

Note in Eq.~(\ref{proof-lma2}) that $\mathcal{P}^{-1}(Y)$ indices can exceed their boundaries. Also, as stated previously, $\mathcal{P}^{-1}(Y)$ should be padded with $\tilde{P} = F - 1 - P$ so that we recover the size of X after transpose convolution. Without loss of generality, we assume $P=0$. We subsequently defined the padded input $\overline{\mathcal{P}^{-1}(Y)}$ as:

\begin{equation}
\overline{\mathcal{P}^{-1}(Y)}_{c,p,q}=
\left\{
\begin{array}{ll}
\mathcal{P}^{-1}(Y)_{c, p-F+1, q-F+1} \mbox{  if   } p,q \in [F, N+F-1]\\
0 \mbox{  if   } p,q \in [1, F-1]\cup[N + F, N +2(F-1)]
\end{array}
\right.,
\end{equation}

where N denotes the dimension of $\mathcal{P}^-1(Y)$. Finally Eq.~(\ref{proof-lma2}) can conveniently be rewritten as:

\begin{equation}
\frac{\partial \Phi}{\partial X_{c,p,q}} = \left(\tilde{W}*\overline{\mathcal{P}^{-1}(Y)}\right)_{p,q}.
\end{equation}

For the sake of readability, the padding is implicitly assumed whenever transpose convolution is performed so that we drop the bar notation. 

We can now proceed to proving Eq.~(\ref{lma3}). We have:

\begin{align*}
    \frac{\partial \Phi}{\partial W_{c',c, r,s}}&=\sum_{c'',x,y}\frac{\partial \Phi}{\partial Z_{c'',x,y}}\cdot\frac{\partial Z_{c'',x,y}}{\partial W_{c', c, r, s}}\\
    &=\sum_{c'',x,y}\mathcal{P}^{-1}(Y)_{c'',x,y}\cdot\frac{\partial}{\partial W_{c', c, r, s}}\left (\sum_{k, r', s'} W_{c'', k, r', s'}X_{k, x+r'-1, y+s'-1}\right)\\
    &=\sum_{x,y}\mathcal{P}^{-1}(Y)_{c',x,y}\cdot X_{c, r + x - 1, s + y - 1} \\
    &= \left (\mathcal{P}^{-1}(Y)*X\right)_{c', c, r, s}
\end{align*}

Finally, proving Eq.~(\ref{lma4}) is straightforward.

\end{proof}

\paragraph{Experiment: theorem demonstration on MNIST.}
We have implemented an architecture with 2 convolution-pooling layers and 1 fully connected layer. The first and second convolution layers are made up of $5\times 5$ kernels with 32 and 64 feature maps respectively. Convolutions are performed without padding and with stride 1. Pooling is performed with $2\times 2$ filters and with stride 2. 

The experimental protocol is the exact same as the one used on the fully connected layered architecture. The only difference is the activation function that we have used here is $\sigma(x)={\rm max}({\rm min}( x, 1), 0)$ which we shall refer to here for convenience as `hard sigmoid function'. Precise values of the hyperparameters $\epsilon$, T, K, beta are given in Tab.~\ref{table-hyp-th}.

We show on Fig.~\ref{th-conv} that $\del{}$ and $-\nab{}$ processes qualitatively very well coincide when presenting one MNIST sample to the network. Looking more carefully, we note that some $\del{s}$ processes collapse to zero. This signals the presence of neurons which saturate to their maximal or minimal values, as an effect of the non linearity used. Consequently, as these neurons cannot move, they cannot carry the error signals. We hypothesize that this accounts for the discrepancy in the results obtained with EP on the convolutional architecture with respect to BPTT. 

\begin{figure}[ht!]
\begin{center}
   \includegraphics[scale=0.5]{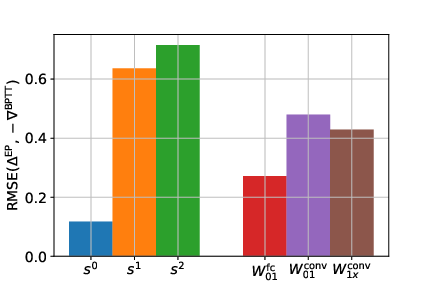}
\end{center}
\caption{RMSE analysis in the convolutional architecture.}
\label{RMSE:conv}
\end{figure}

\newpage

\section{Training experiments (Table \ref{table-results})}
\label{subsec-train}
\paragraph{Simulation framework.} Simulations have been carried out in Pytorch. The code has been attached to the supplementary materials upon submitting this work on the CMT interface. We have also attached a readme.txt with a specification of all dependencies, packages, descriptions of the python files as well as the commands to reproduce all the results presented in this paper. 

\paragraph{Data set.} Training experiments were carried out on the MNIST data set. Training set and test set include 60000 and 10000 samples respectively.
\paragraph{Optimization.} Optimization was performed using stochastic gradient descent with mini-batches of size 20. For each simulation, weights were Glorot-initialized. No regularization technique was used and we did not use the persistent trick of caching and reusing converged states for each data sample between epochs as in \citep{Scellier+Bengio-frontiers2017}. 
\paragraph{Hyperparameter search for EP.} We distinguish between two kinds of hyperparameters: the recurrent hyperparameters - i.e. $T$, $K$ and $\beta$ - and the learning rates. A first guess of the recurrent hyperparameters $T$ and $\beta$ is found by plotting the $\del{}$ and $\nab{}$ processes associated to synapses and neurons to see qualitatively whether the theorem is approximately satisfied, and by conjointly computing the proportions of synapses whose $\del{W}$ processes have the same sign as its $\nab{W}$ processes. $K$ can also be found out of the plots as the number of steps which are required for the gradients to converge. Morever, plotting these processes reveal that gradients are vanishing when going away from the output layer, i.e. they lose up to $~10^{-1}$ in magnitude when going from a layer to the previous (i.e. upstream) layer. We subsequently initialized the learning rates with increasing values going from the output layer to upstreams layers. The typical range of learning rates is $[10^{-3}, 10^{-1}]$, $[10, 1000]$ for T, $[2, 100]$ for K and $[0.01, 1]$ for $\beta$. Hyperparameters where adjusted until having a train error the closest to zero. Finally, in order to obtain minimal recurrent hyperparameters - i.e. smallest T and K possible, both in the energy-based and prototypical setting for a fair comparison - we progressively decreased T and K until the train error increases again. 
\paragraph{Activation functions, update clipping.} For training, we used two kinds of activation functions:

\begin{itemize}
    \item $\sigma(x)=\frac{1}{1+\exp(-4(x-1/2))}$. Although it is a shifted and rescaled sigmoid function, we shall refer to this activation function as `sigmoid'.
    \item $\sigma(x)={\rm max}({\rm min}( x, 1), 0)$. It is the `hard' version of the previous activation function so that we call it here for convenience `hard sigmoid'.
\end{itemize}

The sigmoid function was used for all the training simulations except the convolutional architecture for which we used the hard sigmoid function - see Table~\ref{table-hyp-training}. Also, similarly to \citep{Scellier+Bengio-frontiers2017}, for the energy-based setting we clipped the neuron updates between 0 and 1 so that at each time step, when an update $\Delta s$ was prescribed, we have implemented: $s \gets {\rm max}({\rm min}(s + \Delta s, 1), 0)$. 
\paragraph{Benchmarking EP with respect to BPTT.} In order to compare EP and BPTT directly, for each simulation trial we used the same weight initialization to train the network with EP on the one hand, and with BPTT on the other hand. We also used the same learning rates, and the same recurrent hyperparameters: we used the same $T$ for both algorithms, and we truncated BPTT to $K$ steps, as prescribed by the theory.

\begin{algorithm}[H]{\emph{Input}: static input $x$, parameter $\theta$, learning rate $\alpha$. \\
\emph{Output}: parameter $\theta$.}
\caption{Discrete-time Equilibrium Propagation (EP)}\label{alg-ep}
\begin{algorithmic}[1]
\While{$\theta$ not converged}
\For{each mini-batch x}
\State$\Delta \theta \gets 0$
\For{$t\in[1, T]$}
\State $s_{t+1}\gets \frac{\partial \Phi}{\partial s}(x, s_{t}, \theta)$\Comment{1\textsuperscript{st} phase: common to EP and BPTT}
\EndFor
\For{$t\in[1, K]$}
\State $s_{t+1}^{\beta}\gets \frac{\partial \Phi^{\beta}}{\partial s}(x, s_{t}, \theta)$\Comment{2\textsuperscript{nd} phase: \emph{forward-time} computation}
\State $\del{\theta}\gets\frac{1}{\beta}\left (\frac{\partial \Phi}{\partial \theta}(x, s_{t+1}^{\beta}, \theta) - \frac{\partial \Phi}{\partial \theta}(x, s_{t}^{\beta}, \theta) \right)$
\State $\Delta \theta \gets \Delta \theta + \del{\theta}$
\EndFor
\State $\theta \gets \theta + \alpha\Delta\theta$
\EndFor
\EndWhile
\end{algorithmic}
\end{algorithm}

\begin{algorithm}[H]{\emph{Input}: static input $x$, parameter $\theta$, learning rate $\alpha$. \\
\emph{Output}: parameter $\theta$.}
\caption{Backpropagation Through Time (BPTT)}\label{alg-bptt}
\begin{algorithmic}[1]
\While{$\theta$ not converged}
\For{each mini-batch x}
\State$\Delta \theta \gets 0$
\For{$t\in[1, T]$}
\State $s_{t+1}\gets \frac{\partial \Phi}{\partial s}(x, s_{t}, \theta)$\Comment{1\textsuperscript{st} phase: common to EP and BPTT}
\EndFor
\For{$t\in[1, K]$}

\State $\nab{\theta}\gets \frac{\partial \mathcal{L}}{\partial \theta_{T - t}} $ \Comment{2\textsuperscript{nd} phase: \emph{backward-time} computation}
\State $\Delta \theta \gets \Delta \theta + \nab{\theta}$
\EndFor
\State $\theta \gets \theta - \alpha\Delta\theta$
\EndFor
\EndWhile
\end{algorithmic}
\end{algorithm}

\begin{figure}[H]
\begin{center}
   \includegraphics[scale=0.6]{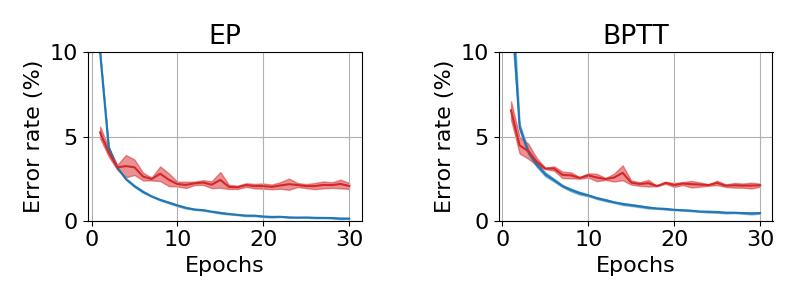}
\end{center}
\caption{Train and test error achieved on MNIST by the fully connected layered architecture with one hidden layer (784-512-10) in the energy-based setting throughout learning, over five trials. Plain lines indicate mean, shaded zones delimiting mean plus/minus standard deviation.}
\end{figure}

\begin{figure}[H]
\begin{center}
   \includegraphics[scale=0.6]{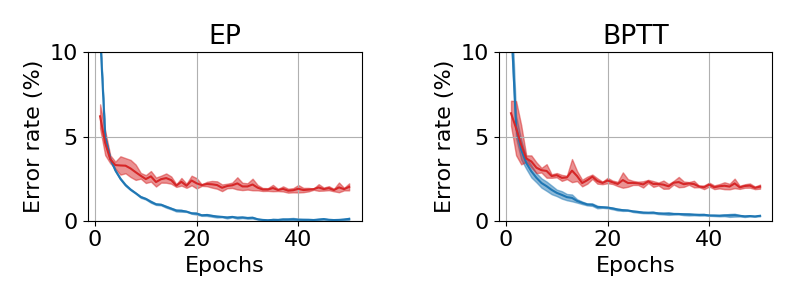}
\end{center}
\caption{Train and test error achieved on MNIST by the fully connected layered architecture with two hidden layers (784-512-512-10) in the energy-based setting throughout learning, over five trials. Plain lines indicate mean, shaded zones delimiting mean plus/minus standard deviation.}
\end{figure}

\begin{figure}[H]
\begin{center}
   \includegraphics[scale=0.6]{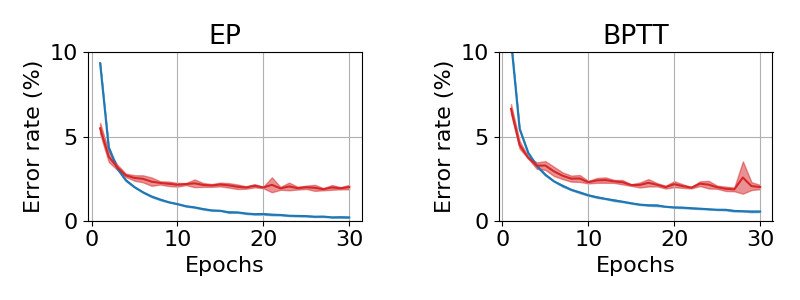}
\end{center}
\caption{Train and test error achieved on MNIST by the fully connected layered architecture with one hidden layer (784-512-10) in the prototypical setting throughout learning, over five trials. Plain lines indicate mean, shaded zones delimiting mean plus/minus standard deviation.}
\end{figure}

\begin{figure}[H]
\begin{center}
   \includegraphics[scale=0.6]{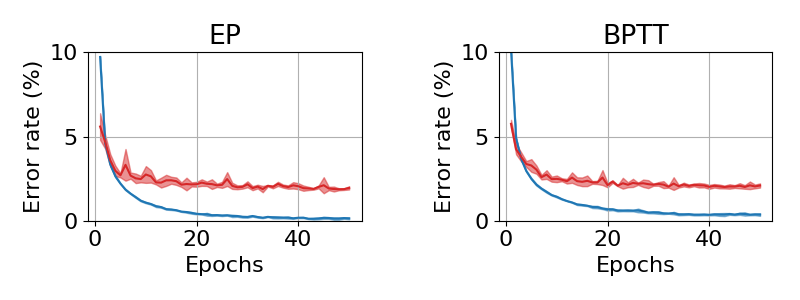}
\end{center}
\caption{Train and test error achieved on MNIST by the fully connected layered architecture with two hidden layers (784-512-512-10) in the prototypical setting throughout learning, over five trials. Plain lines indicate mean, shaded zones delimiting mean plus/minus standard deviation.}
\end{figure}

\begin{figure}[H]
\begin{center}
   \includegraphics[scale=0.6]{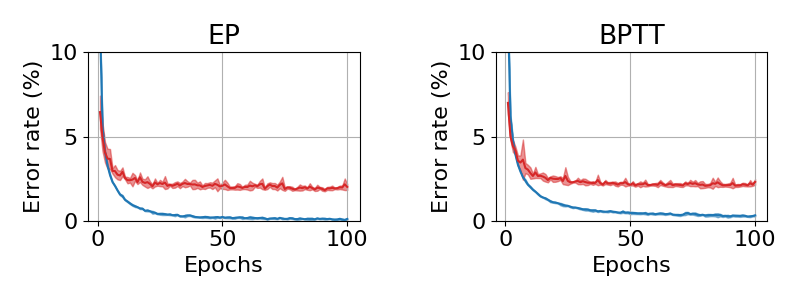}
\end{center}
\caption{Train and test error achieved on MNIST by the fully connected layered architecture with three hidden layers (784-512-512-512-10) in the prototypical setting throughout learning, over five trials. Plain lines indicate mean, shaded zones delimiting mean plus/minus standard deviation.}
\end{figure}

\begin{figure}[H]
\begin{center}
   \includegraphics[scale=0.6]{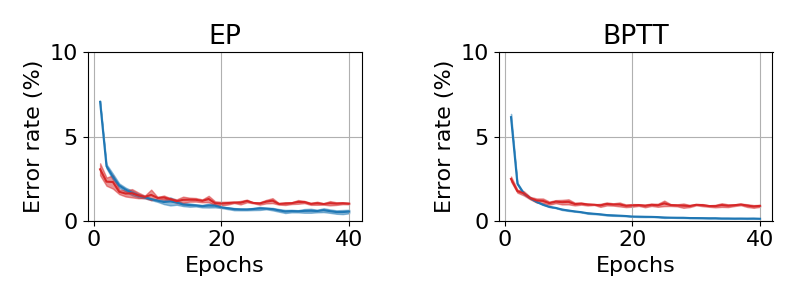}
\end{center}
\caption{Train and test error achieved on MNIST by the convolutional architecture in the prototypical setting throughout learning, over five trials. Plain lines indicate mean, shaded zones delimiting mean plus/minus standard deviation.}
\end{figure}

\newpage

\begin{table}
\begin{center}
    \caption{Table of hyperparameters used to demonstrate Theorem \ref{thm:main}. "EB" and "P" respectively denote "energy-based" and "prototypical", "-$\#$h" stands for the number of hidden layers.}
    \label{table-hyp-th}
\begin{tabular}{@{}lcccccc@{}} \toprule
{}&{Activation}&{T}&{K}&{$\beta$}&{$\epsilon$}\\
\midrule
  {Toy model} &{tanh}&{5000}&{80}&{0.01}&{0.08} \\ 
   \midrule
{EB-1h}&{tanh}&{800}&{80}&{0.001}&{0.08} \\ 
   \midrule
  {EB-2h} &{tanh}&{5000}&{150}&{0.01}&{0.08} \\ 
   \midrule
  {EB-3h} &{tanh}&{30000}&{200}&{0.02}&{0.08} \\ 
   \midrule   
{P-1h}&{tanh}&{150}&{10}&{0.01}&{-} \\ 
   \midrule
  {P-2h}&{tanh}&{1500}&{40}&{0.01}&{-} \\ 
   \midrule
  {P-3h} &{tanh}&{5000}&{40}&{0.015}&{-} \\ 
   \midrule
  {P-conv} &{hard sigmoid}&{5000}&{10}&{0.02}&{-} \\   
\bottomrule
\end{tabular}
\end{center}
\end{table}

\begin{table}

    \begin{center}
    \caption{Table of hyperparameters used for training. "EB" and "P" respectively denote "energy-based" and "prototypical", "-$\#$h" stands for the number of hidden layers.}
    \label{table-hyp-training}
\begin{tabular}{@{}lcccccccc@{}} \toprule
{}&{Activation}&{T}&{K}&{$\beta$}&{$\epsilon$}&{Epochs}&{Learning rates}\\
\midrule
{EB-1h}&{sigmoid}&{100}&{12}&{0.5}&{0.2}&{30}&{0.1-0.05} \\ 
   \midrule
{EB-2h}&{sigmoid}&{500}&{40}&{0.8}&{0.2}&{50}&{0.4-0.1-0.01} \\ 
   \midrule
{P-1h}&{sigmoid}&{30}&{10}&{0.1}&{-}&{30}&{0.08-0.04} \\ 
   \midrule
{P-2h}&{sigmoid}&{100}&{20}&{0.5}&{-}&{50}&{0.2-0.05-0.005} \\ 
   \midrule
{P-3h}&{sigmoid}&{180}&{20}&{0.5}&{-}&{100}&{0.2-0.05-0.01-0.002} \\ 
   \midrule
{P-conv}&{hard sigmoid}&{200}&{10}&{0.4}&{-}&{40}&{0.15-0.035-0.015} \\   
\bottomrule
\end{tabular}
    \end{center}
\end{table}


\end{document}